\documentclass{article} 
\usepackage{iclr2027_conference,times}

\usepackage{amsmath,amsfonts,bm}

\def\eqref#1{equation~\ref{#1}}

\def\1{\bm{1}}

\DeclareMathAlphabet{\mathsfit}{\encodingdefault}{\sfdefault}{m}{sl}
\SetMathAlphabet{\mathsfit}{bold}{\encodingdefault}{\sfdefault}{bx}{n}

\newcommand{\E}{\mathbb{E}}

\usepackage{url}
\usepackage{amsmath,amssymb,amsthm}
\usepackage{booktabs}
\usepackage[protrusion=true,expansion=false]{microtype}
\usepackage{hyperref}
\usepackage{xcolor}
\usepackage{multirow}
\usepackage{graphicx}
\usepackage{algorithm}
\usepackage[T1]{fontenc}
\usepackage{algpseudocode}
\usepackage{tikz}
\usepackage{pgfplots}
\usepackage{diagbox}
\usepackage{wrapfig}
\usepackage{colortbl}

\pgfplotsset{compat=1.18}
\usetikzlibrary{arrows.meta,calc,decorations.pathreplacing,positioning}

\newtheorem{proposition}{Proposition}
\newtheorem{theorem}{Theorem}
\newtheorem{lemma}{Lemma}

\theoremstyle{remark}
\newtheorem{remark}{Remark}

\newcommand{\Lap}{\mathrm{LapSum}}
\providecommand{\E}{\mathbb{E}}

\usepackage[nameinlink]{cleveref}
\crefname{figure}{Fig.}{Figs.}
\crefname{figure}{Figure}{Figures}
\crefname{table}{Tab.}{Tabs.}
\crefname{table}{Table}{Tables}
\crefname{section}{Sec.}{Secs.}
\crefname{section}{Section}{Sections}
\crefname{appendix}{App.}{Apps.}
\crefname{appendix}{Appendix}{Appendices}
\crefname{equation}{Eq.}{Eqs.}
\crefname{equation}{Equation}{Equations}
\crefname{proposition}{Prop.}{Props.}
\crefname{proposition}{Proposition}{Propositions}
\crefname{assumption}{Assum.}{Assums.}
\crefname{assumption}{Assumption}{Assumptions}
\crefname{definition}{Def.}{Defs.}
\crefname{definition}{Definition}{Definitions}

\title{Fast LapSum: Exact Differentiable Top-$k$ at Million Scale}

\author{Jakub Antczak \\
Wrocław University of Science and Technology\\
Wrocław, Poland\\
\texttt{jakub.antczak@pwr.edu.pl} \\
\And
Joanna Wojciechowicz \\
Wrocław University of Science and Technology\\
Wrocław, Poland\\
\texttt{joanna.wojciechowicz@pwr.edu.pl}
\And
Kamil Książek \\
Centre for Credible Artificial Intelligence \\
Warsaw University of Technology \\
Warsaw, Poland\\
\texttt{kamil.ksiazek@pw.edu.pl}
\And
Marcin Mazur \\
Jagiellonian University \\
Kraków, Poland\\
\texttt{marcin.mazur@uj.edu.pl}
\And
Łukasz Struski \\
Jagiellonian University \\
Kraków, Poland\\
\texttt{lukasz.struski@uj.edu.pl}
\And
Jacek Tabor \\
Jagiellonian University \\
Kraków, Poland\\
\texttt{jacek.tabor@uj.edu.pl}
}

\iclrfinalcopy 
\begin{document}

\maketitle
\lhead{Preprint} 

\begin{figure}[h]
    \centering
    \includegraphics[width=0.97\linewidth]{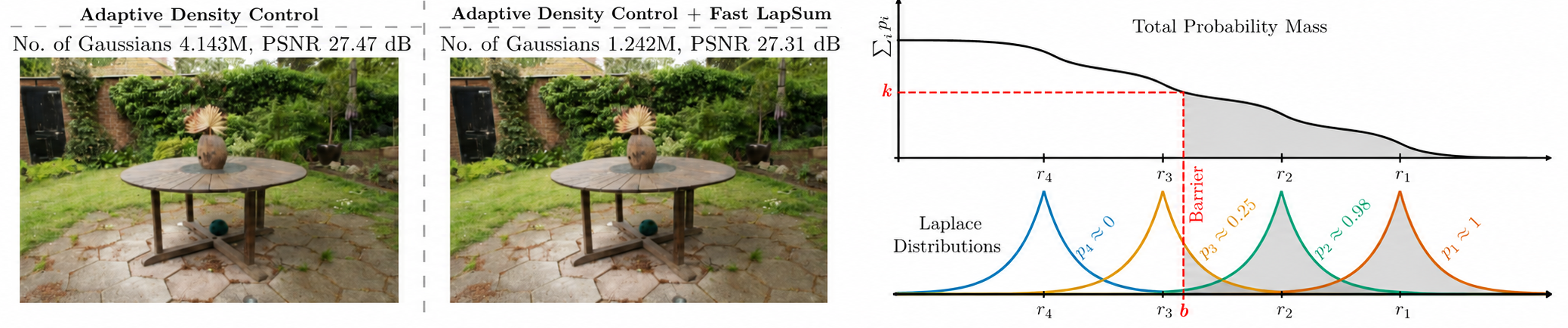}
    \caption{With Fast LapSum, we achieve almost \textbf{two times faster} rendering, using $\textbf{70\%}$ less memory than a baseline in 3D Gaussian Splatting pruning, with nearly no change in visual quality. Our solution scales to \textbf{millions of elements}.\label{fig:placeholder}}
\end{figure}

\begin{abstract}
Selecting the top-$k$ elements is a fundamental operation for inducing sparsity in large-scale models and optimization problems, enabling robust expert activation, token routing or attention pruning. However, hard top-$k$ is non-differentiable, while existing differentiable alternatives become increasingly expensive as the number of coordinates grows. We introduce Fast LapSum, a scalable solver for the LapSum soft top-$k$ formulation that preserves an exact selection mass of $k$, while supporting end-to-end differentiation. In Fast LapSum, we reduce sorting cost using probabilistic bracketing, which restricts sorting to a narrow band of scores around the threshold using a binomial order-statistic from kernel-noised samples. A certification pass upgrades the probabilistic localization to a verified one at the cost of one additional linear pass, with a full-sort fallback that covers the worst-case scenario. Our certified GPU implementation processes $10^6$, $10^7$, and $10^8$ scores in median times of $0.92$, $1.39$, and $7.24$ ms, respectively, making exact-budget soft top-$k$ practical within million-scale optimization loops. We demonstrate this capability in two applications: megapixel sparse adversarial examples with a small fraction of initial image pixels, where Fast LapSum achieves an order-of-magnitude speedup over state-of-the-art methods, and 3D Gaussian splatting. In the latter, we use Fast LapSum to reduce the number of Gaussians produced by Adaptive Density Control to a substantially smaller number while retaining nearly the same rendering quality and massively reducing computation. These results demonstrate that exact-budget differentiable top-$k$ can be incorporated into practical million-scale optimization pipelines.
\end{abstract}

\section{Introduction}
Sparse computation has become a principal approach to scaling neural networks. Mixture-of-experts models activate only a few experts, long-context models retain a limited set of tokens or memories, retrieval systems score millions of candidates and return a small shortlist, and sparse attention or graph layers select only a small neighbourhood. These systems share a single operation, top-$k$ selection, which converts a dense score vector into a sparse, lower-cost subset.

Yet that same operation is a persistent obstacle to end-to-end learning. Hard top-$k$ yields the desired sparsity but almost no useful gradient through the decision boundary. Soft relaxations restore gradients, but often at the cost of dense matrices, transport solves, sorting-like complexity, or masks whose total mass is not directly controllable. For sparse neural systems this is the wrong trade-off: a differentiable sparsifier that is itself expensive merely relocates the bottleneck. We argue that differentiable top-$k$ should be treated as a systems primitive, keeping the exact budget of hard top-$k$, exposing useful gradients, and running cheaply enough to sit inside the training loop. This paper provides such a primitive by making the exact LapSum operator~\citep{lapsum2025} GPU-native and million-scale. Given scores $r=(r_1,\ldots,r_{n})\in\mathbb{R}^n$, budget $k$, and temperature $\alpha<0$, LapSum returns probabilities
\begin{equation}
\label{eq:lapsum-main}
p_i=\sigma\!\left(\frac{b-r_i}{\alpha}\right),\qquad
\sum_{i=1}^n p_i=k,
\end{equation}%
\noindent
where $\sigma$ is the Laplace CDF and the scalar threshold $b$ enforces the budget exactly. Throughout, \emph{budget-exact} refers to this normalization of the mask, not to the cardinality of its support, which coincides with $k$ only in the zero-temperature limit (\Cref{sec:operator}).

Previous GPU implementations solve for $b$ by bisection and re-sum all $n$ terms at every step. Fast LapSum, presented in~\Cref{fig:placeholder}, computes the entire threshold curve. After an initial sort, one stable prefix/suffix scan yields the curve, and the threshold follows from an index lookup and a closed-form root. The backward pass reuses the analytical vector--Jacobian product of \citet{lapsum2025} unchanged. This alone gives nearly two orders of magnitude speed-up over the per-query bisection baseline. For million-scale vectors, where sorting itself becomes the bottleneck, we add a binomial order-statistic bracket that localizes the threshold, streams the vector once, and sorts only the uncertain middle region, still recovering the same $b$. The i.i.d.\ retry variant has expected linear running time under the assumptions of \Cref{prop:linear}. The deployed systematic path is evaluated empirically and uses the full-sort solver as an exact $\mathcal{O}(n\log n)$ fallback. The certified implementation stays in the low milliseconds across the tested large-$n$ range and is much more efficient than the reference LapSum.

We evaluate the operator in two settings with millions of coordinates. In 3D Gaussian Splatting, a soft score-learning stage selects an exact hard primitive subset and reduces the deployed representation. In adversarial attack optimization, a hard-forward straight-through estimator changes at most $K$ pixels while Fast LapSum supplies the score gradient. These applications utilize both soft-to-hard pruning and exact-support training. The most closely related concurrent work is DFTopK~\citep{dftopk2025}, which also targets linear-time differentiable top-$k$ but relaxes $\sum_i p_i=k$. We retain the exact normalization and extend the scaling study three orders of magnitude beyond its largest tested size (\Cref{app:dftopk}).

\noindent Our contribution concerns the threshold computation and the associated large-scale pipeline:
\begin{itemize}
    \item \textbf{Fast LapSum is an exact-budget differentiable top-$k$ operator working with expected linear time ($\mathcal{O}(n)$).} In this category, it is the only method that preserves the unrelaxed constraint $\sum_i p_i=k$. 
    \item \textbf{A probabilistic bracketing pipeline for million-scale vectors.} We develop a probabilistic sample-bracketing mechanism that isolates the exact solver to a narrow uncertainty range. Supported by deterministic certification, it guarantees strict budget exactness. The certified implementation is $457\times$ faster than LapSum at $10^9$ scores on an NVIDIA GH200.
    \item \textbf{Fast LapSum is highly practical at the megapixel scale without altering downstream mathematical objectives}. It accelerates high-resolution pixel-$L_0$ adversarial attacks and enables extreme 3D Gaussian Splatting pruning (even 30\% retention), doubling rendering throughput with negligible quality degradation.
\end{itemize}

\section{Related Work}
\textbf{Differentiable Ranking and Top-$k$.}
Continuous relaxations of sorting and ranking include NeuralSort \citep{neuralsort2019}, SoftSort~\citep{softsort2020}, differentiable sorting networks~\citep{diffsort2021}, the permutahedron/isotonic operators of \citet{blondel2020}, the entropic optimal-transport formulation of \citet{xie2020softtopk}, and convex-analysis sparse top-$k$~\citep{sander2023topk}. These return richer objects (soft ranks, soft permutations, projections) at $\mathcal{O}(n^2)$ or iterative cost, whereas our target is narrower and more operational: a budgeted soft mask with the same output shape as hard top-$k$. We build directly on LapSum~\citep{lapsum2025}, which unifies ranking, sorting and top-$k$ through a single Laplace-CDF threshold equation. Concurrently, DFTopK~\citep{dftopk2025} also targets linear-time differentiable top-$k$ for large-scale recommendation by relaxing the normalization constraint: it places the soft threshold at the midpoint $c=\tfrac12(r_{(k)}+r_{(k+1)})$ of the two boundary order statistics, which is fast and linear-time but makes the mask non-budget-exact at finite temperature and lets the threshold gradient reach only $r_{(k)}$ and $r_{(k+1)}$. \Cref{prop:zerotemp} precises the relation, and we compare both approaches directly in~\Cref{sec:fast-exp} and \Cref{app:dftopk}.

\noindent \textbf{Sampling-Based Selection.}
Locating an order statistic through a small sample is a classical technique~\citep{blum1973time,floyd1975expected,alabi2012fast}. While concurrent work like Prof-K~\citep{profk2026} applies this to hard top-$k$ discarding of out-of-bracket elements, our budget-exact soft operator must accumulate mass ($\sum_i p_i=k$) from all scores. Consequently, sampling only serves to localize the region where the exact budget equation is solved (\Cref{prop:linear}).

\noindent \textbf{Where Differentiable Top-$k$ Is Used.}
A budgeted differentiable selector is the shared mechanism behind many sparse systems, which nowadays mostly fall back on hard selection or task-specific penalties: mixture-of-experts routing under a fixed capacity~\citep{shazeer2017moe,fedus2022switch}, adaptive token pruning~\citep{rao2021dynamicvit}, top-$k$ sparse autoencoders~\citep{makhzani2014ksparse,gao2024topksae}, network and movement pruning~\citep{han2015,frankle2019lottery,movement2020}, and retrieval, sparse attention or gradient compression at larger candidate scale. In each case, the active set is chosen by a hard top-$k$ that blocks gradients while enforcing a hard budget, precisely the regime an exact-budget, GPU-native, million-scale differentiable top-$k$ is built for. Our experiments evaluate the operator and the two large-scale applications described in \Cref{sec:experiments}.

\section{Background: The LapSum Operator}
\label{sec:operator}
In this section, we restate the LapSum operator described in~\citep{lapsum2025} and then define the terminology and~\Cref{prop:zerotemp}. The hard top-$k$ operation relies on a quantile threshold and is harmful during neural network training because moving one score by an arbitrarily small amount can change its mask value discontinuously. LapSum keeps the same logic but introduces a probabilistic test. Each score is perturbed by a Laplace kernel, and $p_i$ is the analytically computed probability that the perturbed copy of $r_i$ falls on the top-$k$ side of the threshold.

Let $t=|\alpha|>0$ denote the temperature where the convention with $\alpha<0$ follows \citep{lapsum2025}. Given scores $r\in\mathbb{R}^n$, define the \emph{budget curve} $\Psi_r:\mathbb{R}\to(0,n)$ by
\begin{equation}
\label{eq:budget-curve}
\Psi_r(b)=\sum_{i=1}^{n}\sigma\!\left(\frac{r_i-b}{t}\right),
\qquad
\sigma(u)=
\begin{cases}
\tfrac12 e^{u}, & u\le 0,\\
1-\tfrac12 e^{-u}, & u>0,
\end{cases}
\end{equation}
where $\sigma$ is the Laplace CDF. Each summand is continuous and strictly decreasing in $b$, so $\Psi_r$ is a strictly decreasing bijection onto $(0,n)$ and the \emph{budget equation} $\Psi_r(b)=k$ has a unique root for every $k\in(0,n)$, which we will simply denote by $b$. The LapSum operator is
\begin{equation}
\label{eq:lapsum-def}
\Lap_k:\mathbb{R}^n\to(0,1)^n,
\quad
\Lap_k(r)[i]=p_i=\sigma\!\left(\frac{r_i-b}{t}\right),
\quad \sum_{i=1}^n p_i=k ,
\end{equation}
so a score far above $b$ receives $p_i\approx1$, one far below it $p_i\approx0$, one at the threshold $p_i=\tfrac12$. Since $\sigma$ is the CDF of $\mathrm{Laplace}(0,t)$, the same object may be written as
\begin{equation}
\label{eq:prob-reading}
p_i=\Pr\big(r_i+\varepsilon\ge b\big),\ \ \varepsilon\sim\mathrm{Laplace}(0,t),
\qquad
\Psi_r(b)=n\,\Pr(Y\ge b),\ \ Y=r_I+\varepsilon,
\end{equation}
with random variable $I$ uniform on $\{1,\dots,n\}$ independent of $\varepsilon$, which identifies $b$ as a quantile of $Y$ and is what the large-$n$ solver of \Cref{sec:prob-main} exploits.

The constraint $\sum_ip_i=k$ with $p\in(0,1)^n$ gives the mask total mass $k$, which need not be an integer. That property holds only as $t\to0$. We use \emph{budget-exact} for the normalization and \emph{exact} for the numerical solution of the budget equation. When a $k$-sparse support is required, the forward pass uses hard top-$k$ and the backward pass uses $\Lap_k$ as the straight-through surrogate.

Differentiating the implicit definition $\Psi_r(b)=k$ gives $\partial b/\partial r$ in closed form, and \citep{lapsum2025} obtains the vector--Jacobian product
\begin{equation}
\label{eq:vjp}
(\nabla_r p)^\top g=f\odot\left(g-\langle g,q\rangle\right),
\qquad
q_i=\frac{f_i}{\sum_j f_j},\qquad
f_i=\frac{1}{2t}e^{-|b-r_i|/t}.
\end{equation}
Once $b$ is computed, evaluating this derivative costs only one element-wise pass and one reduction. We reuse this form unchanged in all variants. The entire computational problem is therefore the forward determination of $b$, the sole target of \Cref{sec:solver}.

\noindent\textbf{Zero-temperature limit and the relation to DFTopK.}
DFTopK fixes its threshold at the midpoint of the two boundary order statistics at every temperature, whereas LapSum solves for it. The following elementary statement, proved in \Cref{app:proofs}, specifies the limiting relation.

\begin{proposition}[Zero-temperature limit of the LapSum threshold]
\label{prop:zerotemp}
Let $r\in\mathbb{R}^n$ have distinct entries with order statistics $r_{(1)}>\dots>r_{(n)}$, while $k\in\{1,\dots,n-1\}$, and let $b$ be the unique root of $\Psi_r(b)=k$ at temperature $t>0$. Then $b\to c:=\tfrac12(r_{(k)}+r_{(k+1)})$ and $\Lap_k(r)[i]\to\mathbf 1[r_i>c]$ as $t\to0^+$, and the limiting mask has exactly $k$ ones and $n-k$ zeros. Since $c$ is by definition the DFTopK threshold, the LapSum threshold converges to it.
\end{proposition}

This result concerns thresholds in the limit and does not establish equivalence of the operators. At fixed $t>0$, the masks differ whenever $b\neq c$. The threshold $b$ moves away from the midpoint to account for the non-boundary scores in the budget equation. This displacement preserves $\sum_ip_i=k$ and lets every score receive a gradient through the threshold in \Cref{eq:vjp}. In DFTopK, the threshold gradient reaches only $r_{(k)}$ and $r_{(k+1)}$. \Cref{tab:budget} measures the resulting budget error.

\section[Fast LapSum: Solving the Budget Equation]{Fast LapSum: Solving the Budget Equation}
\label{sec:solver}

The LapSum operator reduces to one scalar task: find the root $b$ of $\Psi_r(b)=k$. Fast LapSum dispatches between two branches with the same output based on the size of the problem being solved. \Cref{sec:guarantees} states the guarantee for each branch, pseudocode for the full-sort and sample-bracketed branches is provided in Appendices B and C, respectively. 

\subsection[Exact Solve After Sorting]{Exact Solve After Sorting}
\label{sec:fullsort}

Let $s_1\ge\dots\ge s_n$ denote the scores sorted in descending order. Evaluating $\Psi_r$ at a node $s_i$ partitions the summation into scores strictly above $s_i$, the score itself, and scores strictly below it. Applying the Laplace CDF transforms each partition into a geometric-type sum, $\Psi_r(s_i)=(i-1)+\tfrac12(1+B_i-A_i)$, where $A_i=\sum_{j<i}e^{-(s_j-s_i)/t}$ and $B_i=\sum_{j>i}e^{-(s_i-s_j)/t}$. Because all exponents are non-positive, computing these sums is numerically stable and avoids the need for log-sum-exp stabilization. Furthermore, they satisfy the affine recurrences, $A_i=w_i(A_{i-1}+1)$ and $B_i=v_i(B_{i+1}+1)$ with $w_i=e^{-(s_{i-1}-s_i)/t}$, $v_i=e^{-(s_i-s_{i+1})/t}$ and $A_1=B_n=0$. Since affine recurrences are associative when expressed through a $2{\times}2$ transition matrix, a single cooperative block scan computes all node values efficiently. Because $\Psi_r(s_i)$ monotonically increases with $i$, the index lookup $i^\star=\min\{i:\Psi_r(s_i)\ge k\}$ brackets the root between two consecutive scores. Within this interval, the budget equation reduces to a quadratic polynomial in $m=e^{(b-s_{i^\star})/t}$,
\begin{equation}
\label{eq:interval-root}
A_{i^\star}m^2-2\big(i^\star-1-k\big)m-\big(1+B_{i^\star}\big)=0,
\qquad
b=s_{i^\star}+t\log m_+ ,
\end{equation}
where $m_+$ denotes its unique positive root. \Cref{app:scan} details the derivation, the cancellation-safe formulation, and the handling of tied scores. The bi-exponential structure underlying \Cref{eq:interval-root} is intrinsic to the LapSum formulation. By materializing the entire budget curve via a single parallel scan, our solver eliminates the need for iterative search procedures. Consequently, the overall algorithmic cost is reduced to just one sort, one scan, one lookup, one quadratic root calculation, and one fused CDF evaluation.

\subsection{Probabilistic Bracketing: Localizing the Threshold}
\label{sec:prob-main}

At the million scale, a full sort dominates computational cost and is inherently wasteful, as scores far from the threshold $b$ require no internal ordering. We separate the localization of $b$ from its exact determination by sampling kernel-noised scores. This allows us to stream the vector once, retaining only an ambiguous middle set $M=\{r_i:a_L\le r_i<a_R\}$ to be solved exactly by the method in \Cref{sec:fullsort}. Therefore, sampling determines only \emph{where} the exact computation occurs. Prof-K~\citep{profk2026} is the hard-selection counterpart of this structure, and \Cref{app:profk} compares the two methods. Provided $b\in[a_L,a_R)$, the scores outside the bracket remain on fixed branches of the Laplace CDF. Their contributions collapse into stable exponential sums, $\tilde S_L=\sum_{r_i<a_L}e^{(a_L-r_i)/\alpha}$ and $\tilde S_R=\sum_{r_i\ge a_R}e^{(r_i-a_R)/\alpha}$, with non-positive exponents. Because these tail terms are exponentials in $b$ sharing the rates of \Cref{eq:node-values}, they fold directly into $A$ and $B$, preserving the quadratic interval root (\Cref{eq:interval-root}) over the restricted sub-interval (\Cref{app:bracketed}).

This bracketing leverages the quantile interpretation of \Cref{eq:prob-reading}: $\Psi_r(b)=k$ means $b$ is the $q$-quantile of the kernel-noised mixture $Y$, where $q=(n-k)/n$. By drawing $K$ samples $y_s=r_{i_s}+\varepsilon_s$, we define the bracket using order statistics $a_L=y_{(j)}$ and $a_R=y_{(l)}$. While the exact solver processes scores in descending order to accumulate top-$k$ masses, the bracketing algorithm sorts samples in ascending order to compute standard order statistics for quantile estimation. Because each indicator $\mathbf 1[y_s\le b]$ is $\mathrm{Bernoulli}(q)$, the count $C=\#\{s:y_s\le b\}$ is $\mathrm{Binomial}(K,q)$, yielding the exact finite-sample coverage probability:
\begin{equation}
\label{eq:coverage-exact}
\Pr\big(a_L\le b<a_R\big)=\Pr\big(j\le C\le l-1\big),
\qquad C\sim\mathrm{Binomial}(K,q),
\end{equation}
This establishes \Cref{thm:binom-bracket}. When both indices lie inside the sample range, we select $j$ and $l$ through a normal approximation with $z=5.16$, giving coverage $1-2\Phi(-z)+\mathcal{O}(K^{-1/2})$ under i.i.d.\ sampling. \Cref{app:binom} provides the exact finite-sample identity and a distribution-free Bernstein bound. The approximation does not apply when either index leaves the sample range. Algorithmic correctness remains independent of this approximation because certification detects a missed bracket.

For efficiency, the deployed fast path replaces i.i.d.\ sampling with systematic sampling, taking a strided view of the score buffer and adding pre-computed kernel noise. Although this sampler is not covered by~\Cref{eq:coverage-exact}, correctness remains independent of the sampler, see \Cref{app:sampler}. Certification requires one $\mathcal{O}(n)$ pass to compute $\Psi_r(a_L)$ and $\Psi_r(a_R)$ exactly, verifying $\Psi_r(a_R)\le k\le\Psi_r(a_L)$ without distributional assumptions. On failure, the margin is doubled, and the attempt is repeated. After four failed attempts, the implementation falls back to a full sort. While~\Cref{prop:linear} establishes expected $\mathcal{O}(n)$ cost for an analyzed variant that draws i.i.d.\ samples on each retry and uses $K=\lceil n^{2/3}\rceil$, it does not establish this bound for the deployed systematic sampler, which is evaluated empirically. However, we demonstrate robust linear scaling in practice. We bound $|M|$ by relating the empirical law of $r$ to $Y$ through the Laplace kernel, as presented in \Cref{lem:transfer}. Crucially, certification and exactness are strictly enforced by an $\mathcal{O}(n)$ verification pass, which bounds the worst-case cost to $\mathcal{O}(\log N)$ complexity, allowing the fast path to operate safely.


\subsection[What Is Guaranteed, and by Which Path]{What Is Guaranteed, and by Which Path}
\label{sec:guarantees}

The words \emph{exact}, \emph{linear-time} and \emph{fast path} refer to three properties under distinct guarantees. \textbf{(P1) Budget exactness} is deterministic when the bracket contains $b$, certification accepts it, or when the full-data fallback is used. The resulting mask satisfies $\sum_i p_i=k$ up to floating-point precision (\Cref{tab:budget}). \textbf{(P2) Threshold localization} is based on probabilistic bounds for i.i.d.\ sampling, as in~\Cref{eq:coverage-exact}, while \Cref{sec:bracket-exp} measures the deployed systematic sampler. Exact certification strictly enforces deterministic acceptance by checking both endpoints. \textbf{(P3) Linear running time} is guaranteed in expectation for the i.i.d.\ variant~\Cref{prop:linear}. In practice, the deployed systematic sampler achieves strictly linear empirical scaling, bounded by a worst-case
$\mathcal{O}(n\log n)$ full-sort fallback. Importantly, all reported benchmarks and downstream experiments (\Cref{sec:experiments}) run the fully certified solver, actively accounting for the cost of endpoint verification and any necessary retries.

The certified implementation scales across three additional orders of magnitude, as reported in \Cref{tab:ks-main} in \Cref{app:speed-extra}. At $N{=}10^6$, the full-sort branch takes $6.392$\,ms and the certified path takes $0.923$\,ms, giving a $6.9\times$ speedup. At $N{=}10^8$, the corresponding times are $1190.9$ and $7.241$\,ms, giving a $164.5\times$ speedup. The certified $N{=}10^8$ runtime has a 95th percentile of $15.576$\,ms and a retry rate of $33.3\%$.

\section{Speed: Fast LapSum at GPU-Primitive Cost}
\label{sec:fast-exp}

All benchmarks run on a single NVIDIA GH200 120GB using PyTorch 2.9.1 and CUDA 12.9 in FP32, reporting median CUDA-event times after warm-up, see details in \Cref{app:implementation}. For certified Fast LapSum, we use 30 repetitions through $N=10^7$, 24 repetitions at $3{\times}10^7$ and $10^8$, and 20 repetitions at $3{\times}10^8$ and $10^9$. The forward--backward measurements use 30 repetitions at every reported size. The remaining operator curves use 20 repetitions below $10^6$, eight from $10^6$ to below $10^7$, and five at larger sizes. \Cref{tab:ablation-main} isolates our algorithmic gains from CUDA-specific optimizations: in $B{=}16$, $n{=}4096$, replacing the 60-step bisection with the full scan gives a $11.8\times$ speedup. Cooperative execution and CDF fusion increase the total speedup to $72.6\times$. Furthermore, our analytical vector--Jacobian product matches centered finite differences in FP64 with a relative error below $2.4{\times}10^{-10}$, compare \Cref{tab:gradcheck}.

\begin{wrapfigure}{l}{0.6\textwidth}
    \begin{center}
    \vspace{-0,4cm}
    \includegraphics[width=0.6\textwidth]{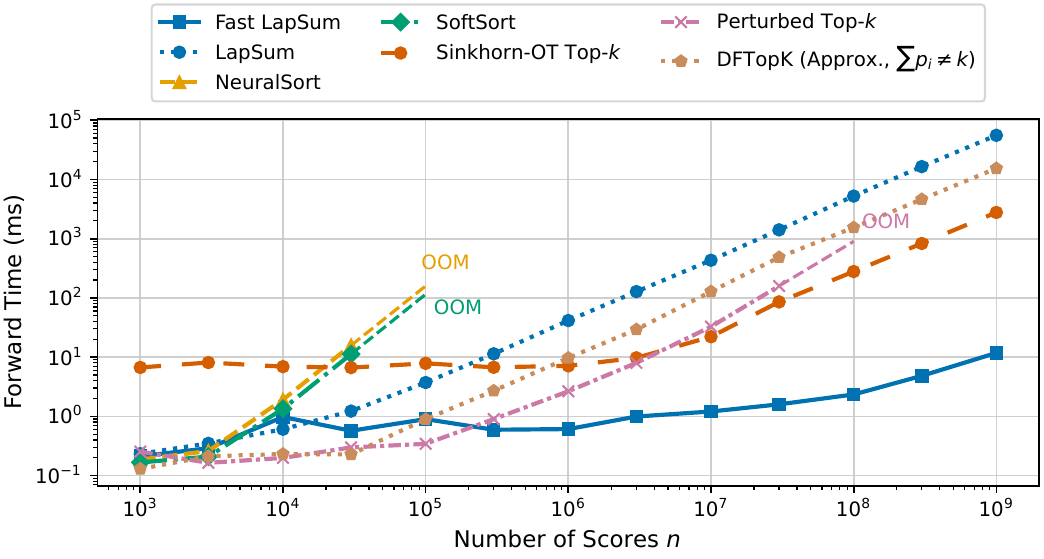}
    \vspace{-0,5cm}
    \caption{Certified forward runtime on an NVIDIA GH200 120GB ($B{=}1$, $k{=}N/16$, FP32). Fast LapSum timings include the endpoint certificate and any retries. Dense sort-matrix methods exceed device memory at scale. At $N=10^8$, certified Fast LapSum has a median runtime of $7.24$\,ms.}
    \label{fig:baseline-plot}
    \end{center}
    \vspace{-0.8\baselineskip}
\end{wrapfigure}

\noindent \textbf{Against Differentiable-Ordering Operators.}
\Cref{fig:baseline-plot} compares certified Fast LapSum with the vanilla LapSum and other differentiable-ordering relaxations in FP32. At $10^8$ and $10^9$ scores, Fast LapSum takes $7.24$ and $121.76$\,ms, while LapSum takes $5240.81$ and $55591.47$\,ms. These values correspond to speedups of $724\times$ and $457\times$, respectively. NeuralSort and SoftSort require quadratic score matrices whose minimum live-tensor estimates exceed GH200 memory from $n=10^5$, while Perturbed Top-$k$ runs out of memory from $n=10^8$. These endpoints are hardware-dependent. Hard top-$k$ is not plotted because it does not provide the soft exact-budget gradient used in our applications.

\noindent \textbf{Budget Accuracy.}
\Cref{tab:budget} in \Cref{app:speed-extra} evaluates the relative budget error $|\sum_i p_i-k|/k$. Fast LapSum maintains exact normalization to FP32 precision across all scales. In contrast, DFTopK's fixed midpoint violates the budget equation with error $2.46$, yielding mass $\approx3.46k$, while other exact-budget relaxations are either out of memory or two orders of magnitude slower. Since competing operators target different objectives, this comparison serves strictly as a compatibility check for the $\sum_ip_i=k$ normalization required by our pipelines, not a general quality ranking.

Forward-backward timings exhibit identical scaling behaviour, as the backward pass simply evaluates the analytical VJP from \cref{eq:vjp}.

\section{Experiments}
\label{sec:experiments}
\subsection{Learned Pruning for 3D Gaussian Splatting}
\label{sec:gs}

\paragraph{Protocol.}
We add a learned pruning stage to the standard 3D Gaussian Splatting Adaptive Density Control (ADC) \citep{kerbl20233d}. ADC creates the candidate Gaussians until iteration $10,000$. Then, densification stops, one score per Gaussian is optimized through an exact-budget soft mask, and the highest-scoring $K$ Gaussians are retained at iteration $12,000$. Training of Gaussian parameters continues on this hard set until iteration $30,000$. We compare vanilla ADC, ADC with Fast LapSum, and a matched hard-pruning control that ranks Gaussians by footprint. Both pruning methods retain the same $K$ and receive the same recovery period. The main evaluation covers all 13 scenes from Mip-NeRF~360, Tanks~\&~Temples, and Deep Blending with three seeds. Implementation details and the score-logging protocol are in \Cref{app:applications}.

\paragraph{Choosing the retention rate.} Before the full evaluation, we select a single global operating point using a separate five-scene and one seed. \Cref{fig:gs-retention} shows the resulting quality--efficiency trade-off as the pruning rate increases. At 70\% pruning, the mean PSNR change relative to full ADC is $-0.065$\,dB; at 90\% pruning, it decreases to $-0.402$\,dB. Based on this trade-off, we set 30\% retention for all subsequent 13-scene and large-scene evaluations.

\begin{figure}[t]
\centering
\includegraphics[width=0.63\textwidth]{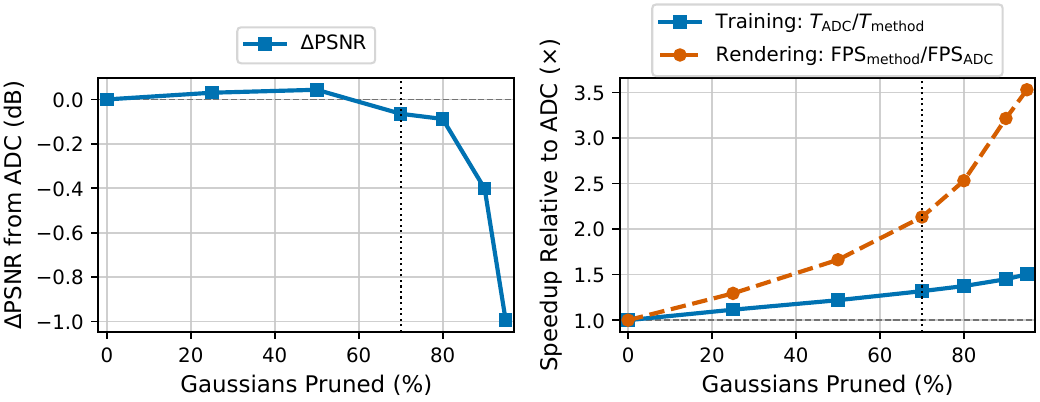}
\caption{Retention sweep averaged across five scenes with one seed. Left: mean PSNR change compared to the vanilla ADC. Right: training $T_{\mathrm{ADC}}/T_{\mathrm{method}}$ and rendering $\mathrm{FPS}_{\mathrm{method}}/\mathrm{FPS}_{\mathrm{ADC}}$ speedups. The dotted line marks the selected operating point of 70\% pruning (30\% retention).}
\label{fig:gs-retention}
\end{figure}

\paragraph{13-scene evaluation.} Using the 30\% retention rate selected above, we evaluate Fast LapSum on all 13 scenes with three seeds per scene. \Cref{tab:gs-standard-quality} reports the quality change relative to the vanilla ADC, and the remaining number of Gaussians, training time, rendering throughput, and model size.

\begin{table}[!t]
\centering
\scriptsize
\setlength{\tabcolsep}{2pt}
\caption{Quality change and resource usage compared to the vanilla ADC at 30\% retention. We first average the three seeds for each scene, then compute the relative difference, and report the mean $\pm$ standard deviation over 13 scenes. A small magnitude indicates that pruning preserves ADC quality.}
\label{tab:gs-standard-quality}
\begin{tabular}{@{}lrrrrrrr@{}}
\toprule
Method & $\Delta$PSNR (dB) $\uparrow$ & $\Delta$SSIM $\uparrow$ & $\Delta$LPIPS $\downarrow$ & $N$ (M) $\downarrow$ & Train (min) $\downarrow$ & FPS $\uparrow$ & Model (MiB) $\downarrow$ \\
\midrule
ADC (reference) & $\mathbf{-}$ & $\mathbf{-}$ & $\mathbf{-}$ & $2.52\pm1.34$ & $12.6\pm2.6$ & $201.9\pm43.5$ & $595.2\pm317.9$ \\
\; +\! hard pruning & $-0.26\pm0.19$ & $-0.009\pm0.006$&$0.022\pm0.010$ & \textbf{\boldmath$0.76\pm0.40$} & $10.3\pm2.4$ & $310.3\pm63.5$ & \textbf{\boldmath$178.5\pm95.2$} \\
\; +\! Fast LapSum & $-0.05\pm0.12$ & $-0.002\pm0.003$ & $0.012\pm0.008$ & \textbf{\boldmath$0.76\pm0.40$} & \textbf{\boldmath$10.0\pm2.3$} & \textbf{\boldmath$385.9\pm57.3$} & \textbf{\boldmath$178.5\pm95.2$} \\
\bottomrule
\end{tabular}
\end{table}

At 30\% retention, Fast LapSum only slightly changes mean PSNR, SSIM, and LPIPS by $-0.046$\,dB, $-0.0020$, and $+0.0117$ relative to the vanilla ADC, despite a significant reduction in the final number of Gaussians, stored model size by 70\%, shortening training time from $12.59$ to $9.97$ minutes ($1.26\times$), and increasing rendering throughput from $201.9$ to $385.9$ FPS ($1.91\times$). Compared with hard pruning, Fast LapSum improves all evaluation metrics. The example of the reconstructed scene by Fast LapSum and ADC is presented in~\Cref{fig:placeholder}.

\begin{table}[!t]
\centering
\scriptsize
\setlength{\tabcolsep}{2,5pt}
\caption{Large-scene quality and efficiency at 30\% retention. Quality changes are paired relative to the matched ADC baseline. Results are means $\pm$ standard deviations over three seeds. Model size (MiB) scales linearly with $N$ and is omitted for brevity. Bold marks the best mean for each scene.}
\label{tab:large-scenes-merged}
\begin{tabular}{@{}llrrrrrr@{}}
\toprule
Scene & Method & $\Delta$PSNR $\uparrow$ & $\Delta$SSIM $\uparrow$ & $\Delta$LPIPS $\downarrow$ & $N$ (M) $\downarrow$ & Train (min) $\downarrow$ & FPS $\uparrow$ \\
\midrule
Mill-19 Bldg.
 & ADC (reference) & $\mathbf{-}$ & $\mathbf{-}$ & $\mathbf{-}$ & $9.51\pm0.12$ & $56.3\pm0.3$ & $73.2\pm4.9$ \\
 & \quad+ Fast LapSum & $-0.17\pm0.11$ & $-0.012\pm0.001$ & $0.017\pm0.000$ & \textbf{\boldmath$2.85\pm0.04$} & \textbf{\boldmath$46.7\pm0.3$} & \textbf{\boldmath$198.4\pm2.0$} \\
 & \quad+ LapSum & $-0.14\pm0.05$ & $-0.010\pm0.000$ & $0.016\pm0.000$ & \textbf{\boldmath$2.85\pm0.04$} & $91.9\pm0.7$ & $197.3\pm1.3$ \\
\addlinespace[2pt]
Rubble (10M)
 & ADC (reference) & $\mathbf{-}$ & $\mathbf{-}$ & $\mathbf{-}$ & $10.74\pm0.03$ & $46.7\pm0.5$ & $105.7\pm1.0$ \\
 & \quad+ Fast LapSum & $-0.25\pm0.05$ & $-0.016\pm0.000$ & $0.026\pm0.000$ & \textbf{\boldmath$3.22\pm0.01$} & \textbf{\boldmath$38.1\pm0.0$} & $236.3\pm4.4$ \\
 & \quad+ LapSum & $-0.26\pm0.11$ & $-0.015\pm0.001$ & $0.026\pm0.000$ & \textbf{\boldmath$3.22\pm0.01$} & $89.7\pm0.5$ & \textbf{\boldmath$237.9\pm2.3$} \\
\addlinespace[2pt]
Rubble (28M)
 & ADC (reference) & $\mathbf{-}$ & $\mathbf{-}$ & $\mathbf{-}$ & $27.73\pm0.33$ & $80.4\pm0.4$ & $49.9\pm4.1$ \\
 & \quad+ Fast LapSum & $+0.06\pm0.15$ & $-0.002\pm0.001$ & $0.009\pm0.001$ & \textbf{\boldmath$8.32\pm0.10$} & \textbf{\boldmath$61.0\pm0.1$} & \textbf{\boldmath$136.5\pm0.5$} \\
 & \quad+ LapSum & $+0.09\pm0.10$ & $-0.002\pm0.001$ & $0.009\pm0.001$ & \textbf{\boldmath$8.32\pm0.10$} & $211.2\pm2.2$ & $136.3\pm0.8$ \\
\bottomrule
\end{tabular}
\end{table}

\paragraph{Large-scene evaluation.}
We next evaluate Fast LapSum on substantially larger 3DGS scenes, where the number of Gaussians reaches tens of millions of elements. This setting verifies the practical benefit of replacing the original LapSum solver with Fast LapSum inside the learned pruning pipeline in demanding conditions.

ADC grows the number of 3D primitives until iteration $30,000$. Score learning continues until iteration $36,000$, when the $K$ highest-scoring Gaussians are retained, and training then continues on the pruned set until iteration $60,000$. In the Rubble scene, the number of Gaussians reaches $10.7$ million in the 10M configuration and $27.6$--$28.1$ million in the 28M configuration, while for Mill-19 Building it is about $9.5$ million.

Across the three configurations, 30\% retention gives a $2.24$--$2.74\times$ rendering speedup and a $1.21$--$1.32\times$ training speedup over ADC, while PSNR changes the range from $-0.25$ to $+0.06$\,dB. Compared to LapSum, Fast LapSum reduces the end-to-end training time by $1.97$--$3.47\times$ while maintaining the same final number of primitives and closely matched quality. In the case of memory size, Fast LapSum-based model needs around $1967.6$ and $674.9$ MiB for Rubble (28M) and Mill-19 Building, respectively, while for ADC these values are much larger ($6558.5$ and $2249.7$ MiB). 

\subsection{Exact-Budget Adversarial Examples}
\label{sec:adv}

\begin{wraptable}{l}{0.47\textwidth}
\vspace{-0,5cm}
\centering
\scriptsize
\setlength{\tabcolsep}{4pt}
\caption{ImageNet attack success rate (\%). Bold marks the best method at each exact budget.}
\label{tab:imagenet-asr}
\begin{tabular}{@{}lrrrrrr@{}}
\toprule
 & \multicolumn{3}{c}{ConvNeXt} & \multicolumn{3}{c}{ResNet-50} \\
\cmidrule(lr){2-4} \cmidrule(l){5-7}
Method & $K{=}10$ & $100$ & $500$ & $K{=}10$ & $100$ & $500$ \\
\midrule
Fast LapSum & \textbf{11.9} & \textbf{78.6} & \textbf{100.0} & \textbf{24.5} & \textbf{92.6} & \textbf{100.0} \\
Sparse-PGD & 7.2 & 47.0 & 94.6 & 15.5 & 66.3 & 98.2 \\
PGD-$L_0$ & 2.3 & 20.3 & 82.5 & 4.8 & 48.2 & 91.8 \\
\bottomrule
\end{tabular}
\end{wraptable}

\paragraph{Threat model and protocol.}

We study untargeted white-box pixel-$L_0$ adversarial attacks, where changing any combination of RGB channels at a single spatial location counts as a single perturbation. The forward pass always uses an exact hard top-$K$ mask, with Fast LapSum supplying the score gradient through the straight-through estimator. We optimize both the mask scores and replacement colors using a Carlini--Wagner margin loss for at most 500 steps and three restarts. Finally, attack success is measured only after the adversarial images are discretized to standard 8-bit color depth. Our primary ImageNet evaluation uses 1,000 labeled validation images classified correctly by both ConvNeXt-V2 Tiny~\citep{convnextv2} and ResNet-50. Sparse-PGD~\citep{zhong2025sparse} and PGD-$L_0$~\citep{croce2019sparse} receive the same fixed support budget and optimization limit.

\begin{wraptable}{l}{0.47\textwidth}
\vspace{-0,5cm}
\centering
\scriptsize
\setlength{\tabcolsep}{3pt}
\caption{Native-resolution DIV2K comparison (10 images, 3 restarts each). ASR is the fraction of images attacked successfully in at least one restart. Total time is the mean duration of one image-restart attack. 
Selector time is the cumulative time spent in performing differentiable selection operations. An attack stops either after success or 500 optimization steps. Bold marks the lower time per budget.}
\label{tab:div2k-backend}
\begin{tabular}{rcrrrr}
\toprule
& & \multicolumn{2}{c}{Total time (s) $\downarrow$} & \multicolumn{2}{c}{Selector time (s) $\downarrow$} \\
\cmidrule(lr){3-4}\cmidrule(lr){5-6}
\multirow{2}{*}{$K$} & ASR, & Fast & \multirow{2}{*}{LapSum} & Fast & \multirow{2}{*}{LapSum} \\
& both (\%) & LapSum & & LapSum & \\
\midrule
100  & 0  & \textbf{30.29} & 61.60 & \textbf{2.16} & 27.31 \\
500  & 0  & \textbf{30.41} & 72.13 & \textbf{2.28} & 37.89 \\
1000 & 40 & \textbf{25.36} & 57.18 & \textbf{2.10} & 29.49 \\
1500 & 60 & \textbf{17.95} & 40.20 & \textbf{2.05} & 18.50 \\
2000 & 80 & \textbf{13.13} & 24.81 & \textbf{1.76} & 11.18 \\
\bottomrule
\end{tabular}
\vspace{-0,75cm}
\end{wraptable}



Fast LapSum achieves the highest measured attack success rate at every reported budget on both classifiers. At $K=100$, the success rates are 78.6\% on ConvNeXt and 92.6\% on ResNet-50, which is presented in \Cref{tab:imagenet-asr}. The corresponding rates are only 47.0\% and 66.3\% for Sparse-PGD, and 20.3\% and 48.2\% for PGD-$L_0$. The mean time for one image and one restart includes optimization and all model evaluations until success or the 500-step limit. On ConvNeXt, these times are 3.31, 6.17, and 10.45 seconds for Fast LapSum, Sparse-PGD, and PGD-$L_0$. On ResNet-50, they are 1.26, 3.11, and 5.83 seconds, respectively.

\paragraph{Native-resolution DIV2K.}
To evaluate our approach at millions of coordinates, we selected ten native-resolution DIV2K images with 2.35--2.99 million pixels. We generate clean top-1 pseudolabels using a frozen ConvNeXt-V2, filter for confidences $\geq 0.8$, and perform untargetted attacks against these labels. Both Fast LapSum and the baseline LapSum share identical initializations, losses, schedules, and forward-pass hard masks, differing only in the backward-pass soft selector.

Because they implement the identical mathematical operator via different threshold solvers, both methods achieve matching attack success sets at every budget. However, Fast LapSum is $1.89$--$2.37\times$ faster end-to-end and $6.35$--$16.63\times$ faster in selector time, compare~\Cref{tab:div2k-backend}. For context on sparsity, 1,500-pixels correspond to only 0.050--0.064\% of the image pixels.

\subsection{Bracketing Robustness}
\label{sec:bracket-exp}

We evaluate the probabilistic bracketing mechanism independently of downstream tasks across a comprehensive grid of vector sizes, i.e. $n\in\{10^5,10^6,10^7\}$, retention ratios $1/2$, $1/16$, and $1/100$, and three score-to-temperature scales. We benchmark the deployed systematic sampler against i.i.d.\ sampling with replacement and sampling without replacement across ten synthetic and real-world distributions: Gaussian, Student-$t_2$, log-normal, Cauchy, a three-cluster mixture, a 16-level tied distribution, Bernoulli, standardized Binomial, learned 3DGS scores and budget-matched bimodal. For each configuration, trials generate a score vector, compute the exact threshold through a full sort, and assess the initial and validated-bracket correctness.

\begin{table}[!t]
\centering
\scriptsize
\setlength{\tabcolsep}{4pt}
\caption{Large-scale evaluation of the systematic sampler ($n{=}10^7$, $k{=}n/16$, $z{=}5.16$, eight trials per row). Coverage is the fraction of initial brackets containing the full-sort threshold. Max. budget error is the post-validation value $\max_{\mathrm{trials}}|\sum_i p_i-k|/k$. Retry rate is omitted because it equals $100\%-\text{Coverage}$. Bold marks the highest coverage.}
\label{tab:bracket-large}
\begin{tabular}{@{}lrr@{\qquad\qquad}lrr@{}}
\toprule
Distribution &
Coverage (\%) &
Max. budget error &
Distribution &
Coverage (\%) &
Max. budget error \\
\midrule
Gaussian
& 75.0
& $1.38{\times}10^{-7}$
& 16-level ties
& \textbf{100.0}
& $2.52{\times}10^{-5}$ \\

Student-$t_2$
& 87.5
& $1.18{\times}10^{-7}$
& Bernoulli
& 87.5
& $7.95{\times}10^{-5}$ \\

Log-normal
& \textbf{100.0}
& $4.42{\times}10^{-8}$
& Binomial(16, 0.25)
& \textbf{100.0}
& $4.22{\times}10^{-5}$ \\

Cauchy
& 62.5
& $7.73{\times}10^{-6}$
& Learned 3DGS
& \textbf{100.0}
& $1.61{\times}10^{-7}$ \\

Three clusters
& \textbf{100.0}
& $1.32{\times}10^{-7}$
& \cellcolor{gray!35}Bimodal (budget-matched)
& \cellcolor{gray!35}\textbf{100.0}
& \cellcolor{gray!35}$5.67{\times}10^{-10}$ \\
\bottomrule
\end{tabular}
\end{table}

On a calibration grid of $z\in \lbrace 1,1.96,3,5.16 \rbrace$, the i.i.d.\ 
sampler yields $69.7\%, 95.5\%, 99.8\%$, and $100\%$ initial coverage, while the systematic sampler achieves $63.6\%, 90.8\%, 99.4\%$, and $100\%$, as presented in \Cref{fig:bracket-coverage}. The maximum validated relative budget error was $7.95\times10^{-5}$, and none of the trials required full-sort fallback. In practice, after some iterations, scores often form a bimodal distribution of suppressed and enhanced scores.  Although the systematic sampler lacks i.i.d.\ guarantees, slightly lowering coverage at small $z$, exact validation ensures correctness. As \Cref{tab:bracket-large} demonstrates, the systematic sampler achieves perfect coverage in them.  Across all ten distributions at $n=10^7$, the initial systematic bracket succeeds in 91.25\% of trials (73/80). The remaining seven cases are resolved by bracket widening, and no trial invokes the full-sort fallback. Thus, in practice, the deployed solver usually remains in its linear-work bracketed regime.

\begin{figure}[!t]
\centering
\includegraphics[width=0.6\textwidth]{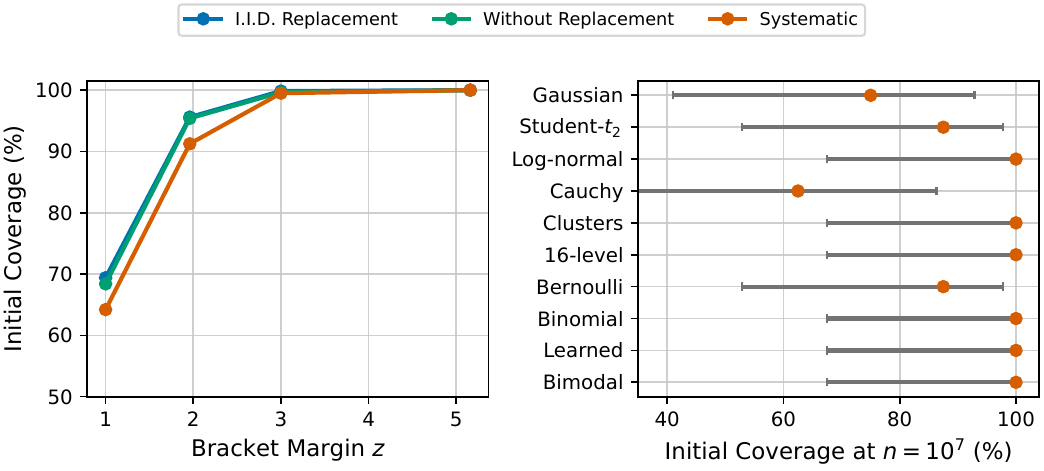}
\caption{Initial bracket coverage before validation. Left: calibration at $n=10^5$ across ten distributions with 1,800 trials per sampler and margin. Right: the deployed systematic sampler at $n=10^7$, $k=n/16$, unit score-to-temperature scale, and $z=5.16$. Error bars are Wilson 95\% confidence intervals over eight trials.} 
\label{fig:bracket-coverage}
\end{figure}

\section{Conclusion}
We presented Fast LapSum, a GPU-native solver for exact-budget differentiable top-$k$ selection. The i.i.d.\ retry variant analyzed in \Cref{prop:linear} has expected linear running time under its stated regularity assumptions. The deployed systematic implementation has empirical scaling, while certification and the exact fallback preserve correctness. At $10^9$ dimensions, the certified implementation is $457\times$ faster than LapSum. As a drop-in replacement, Fast LapSum accelerates megapixel adversarial attacks and enables massive-scale 3D Gaussian Splatting scene pruning, yielding faster and lightweight rendering with negligible quality loss. Ultimately, these results demonstrate that exact-budget differentiable top-$k$ selection is practical for modern, million-parameter machine learning pipelines.

\subsection[Limitations]{Limitations}
\label{sec:limitations}
The construction of our operator is tied to the Laplace kernel. The expected-linear result applies only to the i.i.d.\ retry variant under assumptions (H1)--(H4), including an interior target quantile. It does not cover the deployed systematic sampler or fixed and extremely small budgets as $n$ grows, while some experiments evaluate the latter regime empirically using the certified implementation. Certification and exact fallback preserve correctness, but runtime may approach the $\mathcal{O}(n\log n)$ full-sort solver in the worst-case scenario. Heavy ties can enlarge the middle band. Memory limits and branch crossovers remain hardware-specific.

\subsection*{AI use statement}
In this work, we have not used generative AI tools for formulating mathematical claims, providing critical ingredients for proving mathematical claims, assisting in the writing of proofs, proposing or refining hypotheses, designing or providing feedback on research methodology or experiments, or interpreting results. Furthermore, generating synthetic data sets, assisting with translation, cleaning and reformatting datasets, and supporting qualitative and thematic data analysis are not applicable to this work.

We used generative AI tools for helping develop theoretical models or conceptual frameworks, implementing methods, discovering research topics, identifying relevant literature, summarizing existing literature, and editing the research paper to improve readability. We have reviewed all AI-assisted work. Specifically, any AI-generated code or method implementation was rigorously tested, executed, and manually verified for correctness by the authors. Additionally, all AI-suggested literature was manually verified to ensure the sources exist and accurately reflect the cited content, and all AI-edited text was manually reviewed, refined, and approved by at least two authors. We take responsibility for the final content of this work, including text, claims or artifacts produced with the aid of generative AI.

\subsection*{Reproducibility Statement}
The exact solve is specified by \Cref{eq:node-values,eq:interval-root}, with the cancellation-safe root and tie handling in \Cref{app:scan}. The bracketed solve is given by \Cref{eq:bracket-curve,eq:aug}. Both branches appear as pseudocode in Supplementary Materials. The analytical VJP in \Cref{eq:vjp} is checked against finite differences in \Cref{tab:gradcheck}. \Cref{app:implementation,app:speed-extra,app:dftopk} provide implementation and benchmark details. \Cref{app:applications} specifies the 3DGS and adversarial-attack protocols.

\bibliography{iclr2027_conference}
\bibliographystyle{iclr2027_conference}

\appendix

\section*{Appendix}

\section{Proof of the Zero-Temperature Limit}\label{app:proofs}

We prove \Cref{prop:zerotemp}. To make the temperature explicit, write $\Psi_r$ for the budget curve in \Cref{eq:budget-curve} and $b_t$ for its unique root.

\begin{proposition}
Let $r\in\mathbb R^n$ have distinct entries with order statistics
$r_{(1)}>\cdots>r_{(n)}$, let $k\in\{1,\ldots,n-1\}$, and let
$b_t$ be the unique solution of $\Psi_t(b)=k$. Define
\[
c=\frac{r_{(k)}+r_{(k+1)}}{2},
\qquad
d=\frac{r_{(k)}-r_{(k+1)}}{2}>0,
\]
and
\[
p_i(t)=\sigma\!\left(\frac{r_i-b_t}{t}\right).
\]
Then $b_t\to c$ and $p_i(t)\to\mathbf 1[r_i>c]$ as $t\to0^+$.
Hence the limiting mask has exactly $k$ ones.
\end{proposition}

\begin{proof}
Fix $\varepsilon\in(0,d)$ and define
\[
\lambda_t = e^{-(d-\varepsilon)/t},
\qquad
\mu_t = e^{-(d+\varepsilon)/t}.
\]
Then
\[
\frac{\mu_t}{\lambda_t}=e^{-2\varepsilon/t}\to 0,
\]
so for all sufficiently small $t$ we have $\mu_t<\lambda_t/n$.

Consider first $b=c-\varepsilon$. For the two boundary scores,
\[
r_{(k)}-b=d+\varepsilon,
\qquad
r_{(k+1)}-b=-(d-\varepsilon),
\]
and therefore their contributions to $\Psi_t(b)$ are
\[
1-\frac12\mu_t
\qquad\text{and}\qquad
\frac12\lambda_t,
\]
respectively. For every $i<k$,
$r_{(i)}-b>d+\varepsilon$, so its contribution is at least
$1-\frac12\mu_t$. Hence
\[
\Psi_t(c-\varepsilon)
\ge
k\left(1-\frac12\mu_t\right)
+\frac12\lambda_t
=
k+\frac12(\lambda_t-k\mu_t)
>k,
\]
where the last inequality follows from $k<n$ and
$\mu_t<\lambda_t/n$.

Now take $b=c+\varepsilon$. Then
\[
r_{(k)}-b=d-\varepsilon,
\qquad
r_{(k+1)}-b=-(d+\varepsilon),
\]
so the boundary contributions are
\[
1-\frac12\lambda_t
\qquad\text{and}\qquad
\frac12\mu_t.
\]
The $k-1$ scores above $r_{(k)}$ contribute at most $1$ each, while
every score below $r_{(k+1)}$ contributes at most
$\frac12\mu_t$. Therefore
\[
\Psi_t(c+\varepsilon)
\le
(k-1)
+\left(1-\frac12\lambda_t\right)
+\frac{n-k}{2}\mu_t
=
k-\frac12\bigl(\lambda_t-(n-k)\mu_t\bigr)
<k.
\]

Since $\Psi_t$ is strictly decreasing and
$\Psi_t(b_t)=k$, it follows that
\[
c-\varepsilon < b_t < c+\varepsilon
\]
for all sufficiently small $t$. As $\varepsilon\in(0,d)$ was
arbitrary,
\[
b_t\to c.
\]

It remains to consider the mask. If $i\le k$, then
\[
r_i-b_t \ge r_{(k)}-b_t \to d>0,
\]
and therefore
\[
\frac{r_i-b_t}{t}\to+\infty.
\]
Thus $p_i(t)\to1$. Similarly, if $i\ge k+1$,
\[
r_i-b_t \le r_{(k+1)}-b_t \to -d<0,
\]
so
\[
\frac{r_i-b_t}{t}\to-\infty
\]
and $p_i(t)\to0$. Hence
\[
p_i(t)\to \mathbf{1}[r_i>c].
\]
Since $c$ lies strictly between $r_{(k)}$ and $r_{(k+1)}$, this
limiting mask has exactly $k$ ones.
\end{proof}

\section{Derivation of the Full-Sort Solver}
\label{app:scan}

\begin{algorithm}[h]
\small
\caption{Fast LapSum: Exact Full-Sort Branch.}
\label{alg:fullsort}
\begin{algorithmic}[1]
\Require scores $r\in\mathbb{R}^n$, budget $k\in(0,n)$, temperature $t>0$
\State $s\gets\textsc{SortDescending}(r)$ \Comment{$\mathcal{O}(n\log n)$}
\State $A\gets\textsc{Scan}(w)$, $B\gets\textsc{ReverseScan}(v)$ by \Cref{eq:recurrences}
       \Comment{one cooperative block scan}
\State $\Psi_i\gets (i-1)+\tfrac12(1+B_i-A_i)$ for all $i$
       \Comment{increasing in $i$}
\State $i^\star\gets\min\{i:\Psi_i\ge k\}$
       \Comment{binary search}
\State $b\gets s_{i^\star}+t\log m_+$ with $m_+$ the positive root of \Cref{eq:interval-root}
       \Comment{closed form}
\State \Return $p_i=\sigma\big((r_i-b)/t\big)$ for all $i$
       \Comment{fused with the scan kernel}
\end{algorithmic}
\end{algorithm}

The full-sort solver evaluates the budget curve at the sorted scores and uses these values to locate the interval containing the threshold. \Cref{fig:lapsum-idea} illustrates the underlying LapSum geometry. The two-exponential structure inside an interval is inherited from \citet{lapsum2025}; here we derive the scan form and the closed-form root used by \Cref{alg:fullsort}.

\begin{figure}[h]
\centering
\includegraphics[width=0.8\columnwidth]{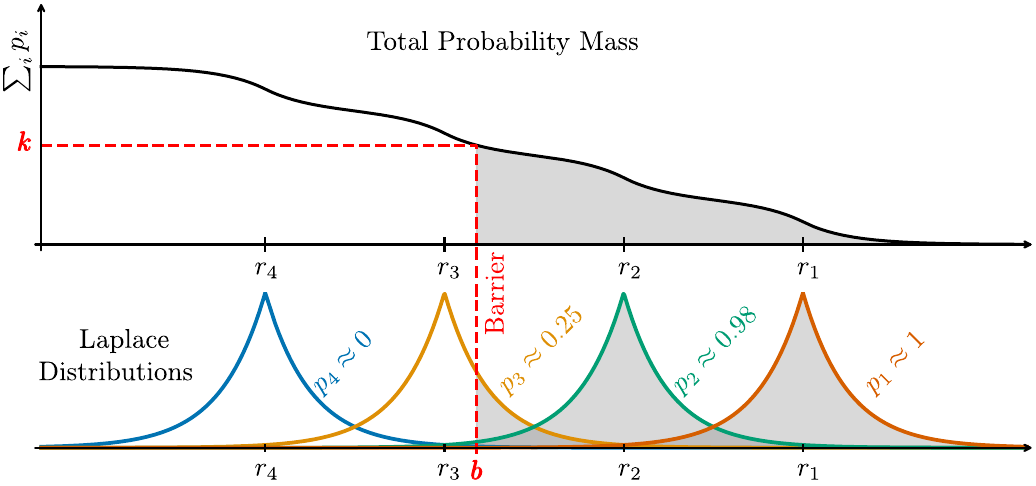}
\caption{LapSum geometry. Each score $r_i$ carries a Laplace kernel of width$t=|\alpha|$, and $p_i=\Pr(r_i+\varepsilon\ge b)$ is the mass above the threshold $b$. The threshold is chosen such that $\sum_i p_i=k$. As $t\to0$, the soft mask converges to hard top-$k$.}
\label{fig:lapsum-idea}
\end{figure}

\noindent\textbf{Node values.}
Let $s_1\ge s_2\ge\dots\ge s_n$ denote the scores sorted in descending order, and recall
\[
\Psi_r(b)=\sum_{j=1}^n \sigma\!\left(\frac{s_j-b}{t}\right).
\]
At $b=s_i$, the terms with $j<i$, $j=i$, and $j>i$ contribute
$1-\tfrac12e^{-(s_j-s_i)/t}$, $\tfrac12$, and
$\tfrac12e^{-(s_i-s_j)/t}$, respectively. Hence
\begin{equation}
\label{eq:node-values}
\Psi_r(s_i)=(i-1)+\tfrac12\big(1+B_i-A_i\big),
\quad
A_i=\sum_{j<i}e^{-(s_j-s_i)/t},
\quad
B_i=\sum_{j>i}e^{-(s_i-s_j)/t}.
\end{equation}

The quantities $A_i$ and $B_i$ satisfy
\begin{equation}
\label{eq:recurrences}
A_i=w_i\,(A_{i-1}+1),\quad
w_i=e^{-(s_{i-1}-s_i)/t},
\qquad
B_i=v_i\,(B_{i+1}+1),\quad
v_i=e^{-(s_i-s_{i+1})/t},
\end{equation}
with $A_1=B_n=0$. All exponents are non-positive, so the recurrences do not require log-sum-exp stabilization. Writing the update $x_i=w_i(x_{i-1}+1)$ as
\[
\begin{pmatrix}
x_i\\
1
\end{pmatrix}
=
\begin{pmatrix}
w_i & w_i\\
0 & 1
\end{pmatrix}
\begin{pmatrix}
x_{i-1}\\
1
\end{pmatrix}
\]
makes the composition associative. A parallel scan therefore computes all $A_i$, and a reverse scan computes all $B_i$. Since $\Psi_r$ is decreasing in its argument and $s_i$ is non-increasing in $i$, the sequence $\Psi_r(s_i)$ is non-decreasing. Thus
\[
i^\star=\min\{i:\Psi_r(s_i)\ge k\}
\]
locates the interval containing the root.

\noindent\textbf{Interval root.}
Write
\[
b=s_{i^\star}+t\delta,
\qquad
m=e^\delta\ge1,
\]
with $b\le s_{i^\star-1}$. For $j<i^\star$,
\[
e^{-(s_j-b)/t}
=
m\,e^{-(s_j-s_{i^\star})/t},
\]
whereas for $j\ge i^\star$,
\[
e^{(s_j-b)/t}
=
m^{-1}e^{(s_j-s_{i^\star})/t}.
\]
Substituting these terms into the budget equation gives
\begin{equation}
\label{eq:interval-eq}
(i^\star-1)-\tfrac12 m A_{i^\star}
+\tfrac{1}{2m}\big(1+B_{i^\star}\big)=k.
\end{equation}
Multiplying by $2m$ gives \Cref{eq:interval-root} from the main text. Let
$\Delta_\star=i^\star-1-k$. Since $A_{i^\star}\ge0$ and
$1+B_{i^\star}>0$, the quadratic has one positive root,
\begin{equation}
\label{eq:root-stable}
m_+
=
\frac{
\Delta_\star+
\sqrt{\Delta_\star^2+A_{i^\star}(1+B_{i^\star})}
}{
A_{i^\star}
}
=
\frac{
1+B_{i^\star}
}{
\sqrt{\Delta_\star^2+A_{i^\star}(1+B_{i^\star})}
-\Delta_\star
}.
\end{equation}
The two expressions are algebraically identical. We use the first for $\Delta_\star\ge0$ and the second for $\Delta_\star<0$ to avoid subtractive cancellation. The second form also covers $A_{i^\star}=0$, which occurs for $i^\star=1$. Finally,
\[
b=s_{i^\star}+t\log m_+.
\]
\noindent\textbf{Ties.}
Equal scores require no separate case. If $s_{i-1}=s_i$, then $w_i=1$ in \Cref{eq:recurrences}, and the corresponding node values $\Psi_r(s_i)$ are equal. The lookup therefore selects the first index satisfying $\Psi_r(s_i)\ge k$. The interval equation and the returned mask remain unchanged because the mask depends on each score only through $\sigma((r_i-b)/t)$.

\section{Exact Solve Inside a Bracket}
\label{app:bracketed}

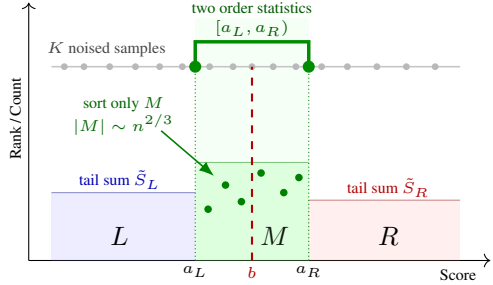
\begin{figure}[thbp]
\centering
\vspace{-0.2\baselineskip}
\begin{tikzpicture}[
  font=\fontsize{6}{7.2}\selectfont,
  arr/.style={-Latex,semithick},
  dot/.style={circle,fill=green!50!black,inner sep=1.0pt},
]
  \draw[->] (0.35,0.38) -- (6.45,0.38) node[pos=0.93, below, inner sep=1pt, yshift=-2pt] {Score};
  \draw[->] (0.35,0.38) -- (0.35,3.8) node[pos=0.5, rotate=90, anchor=center, xshift=5pt, yshift=5pt] {Rank\,/\,Count};

  \fill[blue!45,fill opacity=0.15]  (0.65,0.38) rectangle (2.55,1.28);
  \fill[green!55,fill opacity=0.19] (2.55,0.38) rectangle (4.05,1.68);
  \fill[red!45,fill opacity=0.14]   (4.05,0.38) rectangle (6.05,1.18);
  \draw[blue!55!black,opacity=0.45]  (0.65,1.28) -- (2.55,1.28);
  \draw[green!45!black,opacity=0.45] (2.55,1.68) -- (4.05,1.68);
  \draw[red!55!black,opacity=0.45]   (4.05,1.18) -- (6.05,1.18);
  \node[font=\normalsize] at (1.55,0.7) {$L$};
  \node[font=\normalsize] at (3.6,0.7)  {$M$};
  \node[font=\normalsize] at (5.1,0.7)  {$R$};
  \node[blue!70!black,align=center,inner sep=1pt] at (1.55,1.45)
       {tail sum $\tilde{S}_L$};
  \node[red!70!black,align=center,inner sep=1pt] at (5.1,1.35)
       {tail sum $\tilde{S}_R$};

  \fill[green!35,fill opacity=0.10] (2.55,0.38) rectangle (4.05,3.6);

  \foreach \x/\y in {2.72/1.06,2.95/1.38,3.16/1.16,3.42/1.54,3.72/1.28,3.92/1.48}
     \node[dot] at (\x,\y){};
  \draw[arr,green!50!black] (2.1,2.1) -- (2.8,1.5);
  \node[green!40!black,align=center,inner sep=2pt] at (1.6,2.3)
       {sort only $M$\\$|M|\sim n^{2/3}$};

  \draw[gray!45,semithick] (0.65,2.95) -- (6.05,2.95);
  \foreach \x in {0.82,1.08,1.42,1.76,2.08,2.38,2.74,3.05,3.3,3.58,3.88,4.22,4.56,4.95,5.34,5.72}
    \node[circle,fill=gray!55,inner sep=0.9pt] at (\x,2.95){};
  \node[gray!55!black,anchor=west,inner sep=1pt] at (0.55,3.18)
       {$K$ noised samples};

  \draw[green!50!black,densely dotted] (2.55,0.38) -- (2.55,2.95);
  \draw[green!50!black,densely dotted] (4.05,0.38) -- (4.05,2.95);
  \node[circle,fill=green!55!black,inner sep=1.6pt] at (2.55,2.95){};
  \node[circle,fill=green!55!black,inner sep=1.6pt] at (4.05,2.95){};
  \draw[green!55!black,line width=1.3pt]
       (2.55,2.98) -- (2.55,3.28) -- (4.05,3.28) -- (4.05,2.98);
  \node[green!35!black,anchor=south,align=center,inner sep=1pt] at (3.3,3.31)
       {two order statistics\\$[a_L,a_R)$};

  \draw[red!70!black,thick,dashed] (3.3,0.38) -- (3.3,2.95);
  \node[anchor=north,inner sep=1.5pt] at (2.55,0.36) {$a_L$};
  \node[red!70!black,anchor=north,inner sep=1.5pt] at (3.3,0.36) {$b$};
  \node[anchor=north,inner sep=1.5pt] at (4.05,0.36) {$a_R$};
\end{tikzpicture}%
\caption{The sampled order statistics define thresholds $a_L$ and $a_R$ around the threshold $b$. These thresholds partition the samples into subsets $L$, $M$, and $R$. The tails $L$ and $R$ are aggregated into $\tilde S_L$ and $\tilde S_R$, whereas only the middle subset, $M=\{r_i : a_L \le r_i < a_R\}$, is sorted and used by the exact solver.}
\label{fig:ks-main}
\end{figure}

Assume $b\in[a_L,a_R)$ and partition the scores into
\begin{equation}
L=\{i:r_i<a_L\},\qquad M=\{i:a_L\le r_i<a_R\},\qquad
R=\{i:r_i\ge a_R\}.
\end{equation}
For every $b$ in the bracket, the scores in $L$ stay below $b$ and those in $R$ stay above it, so with the tail sums
\begin{equation}
\label{eq:tails-main}
\tilde S_L=\sum_{r_i<a_L} e^{(a_L-r_i)/\alpha},\qquad
\tilde S_R=\sum_{r_i\ge a_R} e^{(r_i-a_R)/\alpha},
\end{equation}
whose exponents are all non-positive,
\begin{equation}
\label{eq:bracket-curve}
\Psi_r(b)=|R|-\tfrac12\tilde S_R\,e^{-(a_R-b)/t}
+\tfrac12\tilde S_L\,e^{-(b-a_L)/t}
+\sum_{i\in M}\sigma\big((r_i-b)/t\big).
\end{equation}
A single forward pass computes the tail sums, $|R|$, and gathers the middle set.

The bracketed solve then proceeds identically to \Cref{app:scan}, but relies on two augmented coefficients. After sorting $M$ in descending order as $s^M_1\ge\dots\ge s^M_{|M|}$, we write $b=s^M_i+t\delta$ and set $m=e^{\delta}$. Note that the $R$-tail term carries the factor $e^{-(a_R-s^M_i)/t}m$, meaning it naturally joins the $m$-family of the``above'' group, while the $L$-tail term carries $e^{-(s^M_i-a_L)/t}m^{-1}$ and aligns with the ``below'' group. Hence, by defining
\begin{equation}
\label{eq:aug}
A^{\mathrm{aug}}_i=A_i+\tilde S_R\,e^{-(a_R-s^M_i)/t},
\qquad
B^{\mathrm{aug}}_i=B_i+\tilde S_L\,e^{-(s^M_i-a_L)/t},
\qquad
n^{\mathrm{aug}}_i=|R|+i-1,
\end{equation}
the node values become $\Psi_r(s^M_i)=n^{\mathrm{aug}}_i+\tfrac12(1+B^{\mathrm{aug}}_i-A^{\mathrm{aug}}_i)$, and the interval equation takes the quadratic form:
\begin{equation}
A^{\mathrm{aug}}_{i^\star}m^2-2\big(n^{\mathrm{aug}}_{i^\star}-k\big)m-\big(1+B^{\mathrm{aug}}_{i^\star}\big)=0 ,
\end{equation}
which is solved exactly as in \Cref{eq:root-stable} using these augmented coefficients. Since $a_L\le s^M_i\le a_R$, all newly introduced exponents remain non-positive, preserving the numerical stability of the scan. As long as $b\in[a_L,a_R)$, this bracketed formulation evaluates to the exact same threshold as \Cref{alg:fullsort} up to standard floating-point precision, fulfilling the (P1) property discussed in \Cref{sec:guarantees}. Furthermore, boundary cases where $b$ lies above all elements of $M$, below all of them, or $M$ is empty reduce to the same equation with $A^{\mathrm{aug}}$ or $B^{\mathrm{aug}}$ containing only the tail term.

\begin{algorithm}[t]
\small
\caption{Fast LapSum: Sample-Bracketed Branch for large $n$.}
\label{alg:ks}
\begin{algorithmic}[1]
\Require scores $r\in\mathbb{R}^n$, budget $k$, temperature $t>0$, margin $z$, flag \textsc{certify}
\State $q\gets(n-k)/n$; \ \ $K\gets\lceil n^{2/3}\rceil$
\State draw $K$ kernel-noised samples $y_s=r_{i_s}+\varepsilon_s$ and sort them ascending
       \Comment{systematic in the fast path, i.i.d.\ when analyzed}
\State $j,l\gets$ binomial quantiles of $\mathrm{Binomial}(K,q)$ at level $z$, or \Cref{eq:index-choice}
\State if $j < 1$ or $l > K$ then return \Cref{alg:fullsort}
\State $a_L\gets y_{(j)}$,\ \ $a_R\gets y_{(l)}$
       \Comment{no auxiliary solve}
\State stream $r$ once: form $\tilde S_L,\tilde S_R$ of \Cref{eq:tails-main}, count $|R|$, gather $M$
\If{\textsc{certify} \textbf{and} ($\Psi_r(a_R)>k$ \textbf{or} $\Psi_r(a_L)<k$)}
       \Comment{exact, one $\mathcal{O}(n)$ pass}
       \State double the failing side of the bracket and repeat; after four failures set $[a_L,a_R)\gets[r_{\min},r_{\max}]$
\EndIf
\State sort $M$ and run \Cref{alg:fullsort} on $M$ with $A,B$ initialized from $\tilde S_L,\tilde S_R,|R|$
\State \Return $p_i=\sigma\big((r_i-b)/t\big)$ for all $i \in \lbrace 1, ..., n \rbrace$
\end{algorithmic}
\end{algorithm}

Algorithm~\ref{alg:ks} summarizes the complete sample-bracketed Fast LapSum procedure. The probabilistic construction of the bracket and its coverage guarantees are developed in Appendix~\cref{app:binom}.

\section{The Probabilistic Bracket: Guarantees, Sampler, and Runtime}
\label{app:binom}

\subsection*{The Threshold Is a Kernel-Noised Quantile}
For $\alpha<0$ the Laplace CDF gives $\sigma((r_i-b)/t)=\Pr(r_i+\varepsilon\ge b)$ with $\varepsilon\sim\mathrm{Laplace}(0,t)$, so
\begin{equation}
\label{eq:b-as-quantile}
\Psi_r(b)=\sum_i\sigma\!\left(\frac{r_i-b}{t}\right)=n\,\Pr(Y\ge b),
\qquad
Y=r_I+\varepsilon,\quad I\sim\mathrm{Unif}\{1,\dots,n\},
\end{equation}
where $I$ and $\varepsilon$ are independent random variables. Setting $\Psi_r(b)=k$ gives $\Pr(Y<b)=(n-k)/n=:q$, i.e.\ the threshold $b$ is exactly the $q$-quantile of the kernel-noised mixture $Y$.

\subsection*{A Binomial Bracket from Two Order Statistics}
Draw $K$ indices uniformly with replacement and form the kernel-noised sample $y_s=r_{i_s}+\varepsilon_s$ ($\varepsilon_s$ i.i.d.), sorted ascending $y_{(1)}\le\dots\le y_{(K)}$.  By \Cref{eq:b-as-quantile} each draw satisfies $y_s\le b$ independently with probability $q$, so $C=\#\{s:y_s\le b\}$ is exactly binomial and the order statistics bracket $b$.

\begin{theorem}
\label{thm:binom-bracket}
Let $C\sim\mathrm{Binomial}(K,q)$ with $q=(n-k)/n$.  For integers $1\le j<l\le K$,
\begin{equation}
\Pr\big(y_{(j)}\le b< y_{(l)}\big)=\Pr\big(j\le C\le l-1\big).
\end{equation}
\end{theorem}
\begin{proof}
The indicators $\mathbf 1[y_s\le b]$ are i.i.d.\ $\mathrm{Bernoulli}(q)$ by~\Cref{eq:b-as-quantile}, so $C\sim\mathrm{Binomial}(K,q)$.  Now $y_{(j)}\le b$ iff at least $j$ draws fall at or below $b$, i.e.\ $C\ge j$, and $b<y_{(l)}$ iff fewer than $l$ do, i.e.\ $C\le l-1$.  Intersecting the two events gives the claim.
\end{proof}

\noindent\textbf{Choosing the indices.}
While \Cref{thm:binom-bracket} establishes an exact finite-sample identity for coverage, it leaves the choice of the indices $j < l$ open. We consider three distinct strategies for determining these bounds:
\begin{equation}
\label{eq:index-choice}
j=\big\lfloor Kq-z\sqrt{Kq(1-q)}\big\rfloor,\qquad
l=\big\lceil Kq+z\sqrt{Kq(1-q)}\big\rceil .
\end{equation}
For extreme budgets, the sample quantile indices may fall outside \([1,K]\). The implementation clamps the indices to the sample range and extends the affected endpoint by \(10|\alpha|\). This case is outside the sampling-coverage result above. The endpoint certificate accepts a valid bracket, widens a missed bracket, and invokes the exact full-sort fallback after four failed attempts. Consequently, index clipping can affect runtime but not correctness.
\begin{itemize}
\item \emph{Exact binomial quantiles.} Take $j=\max\{u:\Pr(C<u)\le\eta/2\}$ and $l=\min\{u:\Pr(C\ge u)\le\eta/2\}$, evaluated once from the regularized incomplete beta function. \Cref{thm:binom-bracket} then gives coverage $\ge1-\eta$ with no approximation. The cost is a single scalar special-function evaluation per call, which is negligible. This is the setting we recommend when a stated finite-sample guarantee is required.
\item \emph{Bernstein bound.} With $j,l$ as in \Cref{eq:index-choice}, the failure event is $|C-Kq|>z\sqrt{Kq(1-q)}$, and Bernstein's inequality for bounded summands gives the distribution-free bound
\begin{equation}
\label{eq:fail-bound}
\Pr\big(b\notin[a_L,a_R)\big)\;\le\;2\exp\!\left(-\frac{z^2}{2\big(1+\tfrac{z}{3\sqrt{Kq(1-q)}}\big)}\right)
\;=\;2\exp\!\Big(-\tfrac12z^2\big(1-o(1)\big)\Big),
\end{equation}
where the $o(1)$ is explicit and vanishes as $K\to\infty$. We prefer Bernstein's inequality over Hoeffding's because the target quantile $q=(n-k)/n$ is typically far from $1/2$. In this highly skewed regime, the Bernoulli variance $q(1-q)$ is small, rendering Hoeffding's variance-agnostic bound of $2\exp(-2z^2q(1-q))$ overly loose.

\item \emph{Normal approximation.} \Cref{eq:index-choice} with the reading ``coverage $\ge1-2\Phi(-z)$'' is a normal approximation, not a bound. By the Berry--Esseen theorem the standardized binomial CDF is within $\mathcal{O}(K^{-1/2})=\mathcal{O}(n^{-1/3})$ of the normal one, so the correct statement is coverage $=1-2\Phi(-z)+\mathcal{O}(n^{-1/3})$. While this normal approximation is often colloquially cited as a strict guarantee, the finite-sample error term necessitates relying on the exact or Bernstein bounds for formal claims.
\end{itemize}
For our default margin $z=5.16$, the Bernstein bound in \Cref{eq:fail-bound} restricts the theoretical failure probability to strictly below $10^{-5}$. This ensures \textbf{the initial bracket successfully captures $b$ in nearly all instances under i.i.d. sampling, avoiding the need for subsequent bracket widening}. Importantly, because this localization relies entirely on two empirical order statistics, it introduces no auxiliary LapSum overhead. Finally, balancing the $\mathcal{O}(K \log K)$ sorting cost of the sample against the sorting of the localized set $M$ directly motivates our sublinear sample size choice of $K \approx n^{2/3}$; a coarser $K \approx 2\sqrt{n}$ sample is also available when a wider bracket is acceptable.

\subsection[From the Analyzed Sampler to the Deployed One]{From the Analyzed Sampler to the Deployed One}
\label{app:sampler}

\Cref{thm:binom-bracket} assumes $K$ i.i.d.\ draws with replacement. The deployed fast path does not satisfy this assumption, so we analyze the sampling procedures separately.

\textbf{With replacement (analyzed).} The coverage theorem applies verbatim. The runtime result in \Cref{prop:linear} additionally draws a fresh i.i.d.\ sample on every retry. This is the variant to which the theoretical failure probabilities refer.

\textbf{Without replacement.} Drawing $K$ distinct indices uniformly makes $C$ hypergeometric. The identity of \Cref{thm:binom-bracket} still holds with this distribution. By the reduction of \citet{hoeffding1963probability}, every convex tail bound valid for the binomial case remains valid, so \Cref{eq:fail-bound} is unchanged. Sampling without replacement is at least as concentrated under these bounds.

\textbf{Systematic (deployed).} For throughput the fast path takes a strided view of the score buffer, $i_s=(o+s\lceil n/K\rceil)\bmod n$, and adds pre-computed kernel noise, which avoids a gather and a per-call random-number pass. This constitutes systematic sampling, which naturally lacks the independence or exchangeability assumptions required by the theoretical bounds above. The primary failure case arises if the scores are arranged periodically and the period shares a large common factor with the stride, the sample is drawn from a subpopulation whose quantiles differ from those of the full vector, and the bracket can miss $b$ by an arbitrary margin. Sorted or block-structured inputs are realistic instances. We randomize the offset, keep $z=5.16$, and measure coverage on the ten score distributions in \Cref{sec:bracket-exp}. These choices do not constitute a proof. Correctness comes from the certification pass of \Cref{alg:ks}, which directly evaluates both bracket endpoints and widens any missed bracket.

\subsection[Relation to Probabilistic Hard Selection (Prof-K)]{Relation to Probabilistic Hard Selection (Prof-K)}
\label{app:profk}

Sampled thresholds combined with a one-pass filter are a standard technique in selection algorithms, most recently instantiated for hard top-$k$ by the concurrent Prof-K~\citep{profk2026}. While the high-level structures appear similar, a budget-exact soft operator introduces fundamental differences. We highlight four key distinctions that necessitated the design of \Cref{alg:ks}.

\emph{What is being localized.} Prof-K localizes an order statistic of the scores, the $k$-th largest value. Our threshold is not an order statistic. By \Cref{eq:b-as-quantile}, it is the $(n{-}k)/n$ quantile of the kernel-noised mixture $Y=r_I+\varepsilon$. The sample must therefore be noised before sorting, and the binomial bracket uses $q=(n-k)/n$. At temperature $t\to0$, the two thresholds coincide.

\emph{What happens to the discarded elements.} A hard selector may drop everything outside the buffer because those elements cannot enter the answer. A budget-exact soft operator must retain their aggregate contribution because every score contributes mass to $\sum_ip_i=k$. The two exact sums in \Cref{eq:tails-main} summarize the tails, which re-enter the interval equation through the augmented coefficients in \Cref{eq:aug}. This step makes the bracketed solve identical to the full solve.

\emph{Form of the guarantee.} Prof-K returns the correct hard top-$k$ with probability at least $1-\epsilon$ for a user-specified $\epsilon$, distribution-free. Our guarantees are split differently (\Cref{sec:guarantees}): the budget is exact deterministically given a valid bracket, the bracket is probabilistic only in the uncertified fast path, and certification makes it deterministic at the cost of one $\mathcal{O}(n)$ pass, since $\Psi_r(a_L),\Psi_r(a_R)$ can be evaluated exactly, which has no analogue for a hard selector, where verifying a candidate set is as expensive as producing it. Conversely, Prof-K's distribution-free buffer-size analysis is an alternative to the index choice \Cref{eq:index-choice} when a stated $\epsilon$ is preferred to certification, and the two are compatible: any bracket rule can be plugged into \Cref{alg:ks} without touching the solve.

\emph{Consequences for cost.} Prof-K reports $1.5\times$--$10\times$ speedups over optimized hard top-$k$ implementations at large $n$ with small to moderate $k$, i.e.\ the regime where our bracketed branch also pays off. Read together with \Cref{fig:baseline-plot}, this supports the claim of \Cref{sec:fast-exp} that the differentiable operator is within a constant factor of a probabilistic hard selector, and that the constant is not large.

\subsection*{Empirical Coverage}
The theoretical coverage result above applies to i.i.d.\ sampling with replacement. We separately test that sampler, sampling without replacement, and the deployed systematic sampler in \Cref{sec:bracket-exp}. The calibrated coverage curves and the $n=10^7$ large-scale evaluation are reported in \Cref{fig:bracket-coverage,tab:bracket-large}.

\subsection*{Expected Linear Runtime}
\label{app:linear}
In this section, we formally analyze the expected $\mathcal{O}(n)$ time complexity of our bracketed solver under an i.i.d.\ sampling scheme. Because the running time is dominated by the size of the middle set $M$, we must bound the empirical score concentrations using the kernel-noised mixture $Y$. \Cref{lem:transfer} provides this critical mathematical link, motivating hypotheses required by~\Cref{prop:linear} to prove an expected linear runtime. Although our deployed systematic sampler substitutes i.i.d.\ sampling for hardware-optimized reads and is therefore evaluated empirically, this analysis justifies the $\mathcal{O}(n)$ runtime.

\begin{lemma}[Kernel mass transfer]
\label{lem:transfer}
Let $Y=r_I+\varepsilon$ with $\varepsilon\sim\mathrm{Laplace}(0,t)$ as in \Cref{eq:b-as-quantile}. For any interval $\mathcal{I}=[a,a+w)$ with $w>0$,
\begin{equation}
\#\{i:r_i\in\mathcal{I}\}\;\le\;C(w/t)\,n\,\Pr(Y\in\mathcal{I}),
\qquad
C(u)=\frac{2}{u}e^{u},
\end{equation}
and $C$ attains its minimum $2e$ at $u=1$.
\end{lemma}

\begin{proof}
Let $m=\#\{i:r_i\in\mathcal{I}\}$. The density of $Y$ is $f_Y(y)=\frac1n\sum_i\frac{1}{2t}e^{-|y-r_i|/t}$. For $y\in\mathcal{I}$ and $r_i\in\mathcal{I}$ we have $|y-r_i|<w$, so $f_Y(y)\ge\frac{m}{2tn}e^{-w/t}$. Integrating over $\mathcal{I}$ gives $\Pr(Y\in\mathcal{I})\ge\frac{mw}{2tn}e^{-w/t}$, which rearranges to the claim. The stated minimum follows from $C'(u)=0$ at $u=1$.
\end{proof}

\begin{remark}
\label{rem:transfer}
\Cref{lem:transfer} shows that raw scores cannot hide from the noised mixture at scales comparable to the kernel width: for $w\asymp t$ the middle set is at most a constant multiple of $n\Pr(Y\in\mathcal{I})$. It does \emph{not} rule out concentration at scales far below $t$, where the factor $C(w/t)\approx 2t/w$ degrades: if a constant fraction of the scores is tied at a single value inside a bracket of width $w\ll t$, then $|M|=\Theta(n)$ and the bound is tight. This is exactly the configuration that breaks linear time, and it is why \Cref{prop:linear} assumes (H3) and why the fallback of \Cref{alg:ks} matters.
\end{remark}

\begin{proposition}[Expected linear time of the i.i.d.\ retry variant]
\label{prop:linear}
Consider the following analyzed variant of \Cref{alg:ks}. Set $K=\lceil n^{2/3}\rceil$. On attempt $\tau=0,1,\ldots$, draw a fresh i.i.d.\ sample of $K$ kernel-noised scores, set $z_\tau=2^\tau z_0$, and compute
\begin{equation}
\label{eq:retry-indices}
j_\tau=\left\lfloor Kq-z_\tau\sqrt{Kq(1-q)}\right\rfloor,
\qquad
l_\tau=\left\lceil Kq+z_\tau\sqrt{Kq(1-q)}\right\rceil.
\end{equation}
The bracket is accepted only after certification. If either index ceases to be interior, the algorithm invokes the exact full-sort solver. Assume
\begin{itemize}
\item[(H1)] every attempt uses a fresh i.i.d.\ sample from $Y$;
\item[(H2)] the density $f_Y$ is bounded below by $c_1>0$ on a neighbourhood $\mathcal{N}$ containing the brackets considered before fallback;
\item[(H3)] the empirical score distribution is regular on $\mathcal{N}$: there is $c_3<\infty$ with $\#\{i:r_{i}\in \mathcal{I}\}\le c_{3}(n\vert{}\mathcal{I}\vert{}+1)$ for every interval $\mathcal{I}\subseteq\mathcal{N}$;
\item[(H4)] the target quantile is bounded away from both extremes: there exists a constant $\delta>0$ such that $\delta\le q\le1-\delta$.
\end{itemize}
Then the expected running time is $\E[T]=\mathcal{O}(n)$, and the expected total number of middle-set elements processed across all attempts is $\mathcal{O}(n^{2/3})$. Without these assumptions, certification still preserves correctness and the full-sort fallback gives $\mathcal{O}(n\log n)$ worst-case complexity.
\end{proposition}

\begin{proof}
Let $C_\tau$ be the number of samples on attempt $\tau$ that do not exceed the exact threshold. By (H1), $C_\tau\sim\mathrm{Binomial}(K,q)$. The bracket misses the threshold only if $C_\tau<j_\tau$ or $C_\tau\ge l_\tau$. By \Cref{eq:fail-bound}, while the indices remain interior there is a constant $c_{\mathrm B}>0$ such that
\begin{equation}
\label{eq:retry-failure}
\rho_\tau:=\Pr\bigl(b\notin[a_{L,\tau},a_{R,\tau})\bigr)
\le 2\exp(-c_{\mathrm B}z_\tau^2)
=2\exp(-c_{\mathrm B}4^\tau z_0^2).
\end{equation}
This bound is obtained directly in binomial rank space and makes no claim that doubling a score interval exactly doubles a standardized rank margin.

The expected probability mass between the two order statistics on attempt $\tau$ is
\begin{equation}
\E\!\left[F_Y(a_{R,\tau})-F_Y(a_{L,\tau})\right]
=\frac{l_\tau-j_\tau}{K+1}
=\mathcal{O}\!\left(\frac{z_\tau}{\sqrt K}\right).
\end{equation}
By (H2), the expected bracket width is therefore $\mathcal{O}(z_\tau/(c_1\sqrt K))$. Assumption (H3) then gives
\begin{equation}
\label{eq:retry-middle}
\E|M_\tau|=\mathcal{O}\!\left(\frac{c_3}{c_1}\frac{z_\tau n}{\sqrt K}+1\right)
=\mathcal{O}(2^\tau n^{2/3}).
\end{equation}
The density and score-regularity assumptions enter only this middle-set bound; the failure probability in \Cref{eq:retry-failure} is a binomial rank-space statement.

Fresh samples make the conditional failure bound on each reached attempt equal to \Cref{eq:retry-failure}. Consequently,
\begin{equation}
\label{eq:retry-reach}
\Pr(\text{reach attempt }\tau)
\le\prod_{s=0}^{\tau-1}2\exp(-c_{\mathrm B}4^s z_0^2)
=2^\tau\exp\!\left(-c_{\mathrm B}z_0^2\frac{4^\tau-1}{3}\right).
\end{equation}
One attempt requires an $\mathcal{O}(n)$ stream, which inherently includes the deterministic certification check $\Psi_r(a_R)\le k\le\Psi_r(a_L)$, an $\mathcal{O}(K\log K)$ sample sort, and an $\mathcal{O}(|M_\tau|\log n)$ middle-set sort. Combining \Cref{eq:retry-middle,eq:retry-reach},
\begin{align}
\E[T]
&=\mathcal{O}\!\left(
n\sum_{\tau\ge0}\Pr(\text{reach }\tau)
+n^{2/3}\log n\sum_{\tau\ge0}2^\tau\Pr(\text{reach }\tau)
\right)
=\mathcal{O}(n).
\end{align}
Both series converge because the factor in \Cref{eq:retry-reach} decays doubly exponentially. The same calculation without the logarithmic sorting factor gives an expected total middle-set size of $\mathcal{O}(n^{2/3})$.

By (H4), an index in \Cref{eq:retry-indices} first ceases to be interior only when $z_\tau^2=\Theta(K)$. \eqref{eq:retry-reach} bounds the probability of reaching this point by $\exp[-\Omega(K)]$. The expected contribution of the $\mathcal{O}(n\log n)$ full-sort fallback is therefore $\mathcal{O}(n\log n\,e^{-\Omega(K)})=o(n)$.
\end{proof}

While this expected linear-time guarantee requires the i.i.d. hypothesis (H1), the exact certification ensures that substituting a faster systematic sampler in deployment only impacts the empirical fallback rate, never the exactness of the final output.

\paragraph{Scope of the result.}
The proposition applies to the analyzed i.i.d.\ retry-until-certified variant. The deployed implementation uses systematic sampling and a finite retry cap, and is not covered by this expected-linear theorem. Its scaling is established empirically. In both variants, certification rejects a missed bracket, and the exact full-sort fallback preserves correctness with $\mathcal{O}(n\log n)$ worst-case complexity. Negligible density at the threshold violates (H2), heavy ties can violate (H3), and extreme budgets violate (H4). These cases affect the runtime statement, not the validity of an accepted output. On the initial i.i.d.\ attempt, \Cref{eq:retry-middle} gives $\E|M_0|=\mathcal{O}(n^{2/3})$; the implementation caps the gather buffer and falls back when this cap is exceeded.

\section{Implementation Details}
\label{app:implementation}

All forward implementations use the same autograd interface. The reference solver performs 60 vectorized bisection steps. The PyTorch implementation replaces bisection with the full-sort scan in \Cref{alg:fullsort}. The CUDA implementation evaluates the affine scan cooperatively, with one block assigned to each batch element. Its fused version computes the final CDF in the same kernel. All variants use the analytical backward pass in \Cref{eq:vjp}.

\begin{table}[!htbp]
\centering
\small
\setlength{\tabcolsep}{6pt}
\begin{tabular}{lrr}
\toprule
Variant & Forward (ms) & Speedup \\
\midrule
Bisection (60 steps) & 7.614 & $1.0\times$ \\
PyTorch full scan & 0.647 & $11.8\times$ \\
Cooperative CUDA scan & 0.206 & $36.9\times$ \\
Fused CUDA scan and CDF & \textbf{0.105} & \textbf{$72.6\times$} \\
\bottomrule
\end{tabular}
\caption{Median forward runtime on an NVIDIA GH200 120GB ($B=16$, $n=4096$, $k=256$, $\alpha=-1$). Speedup is measured against 60-step bisection. Bold marks the lowest runtime and highest speedup.}
\label{tab:ablation-main}
\end{table}

Replacing bisection with the PyTorch full scan reduces the runtime from $7.614$ to $0.647$\,ms. The cooperative CUDA scan and CDF fusion reduce it further to $0.105$\,ms, a $72.6\times$ speedup over bisection.

The bracketed implementation first samples and perturbs scores with Laplace noise. Sorting this sample gives the endpoints $a_L=y_{(j)}$ and $a_R=y_{(l)}$. One pass over the input computes the two tail sums, evaluates the budget curve at both endpoints, and gathers the middle set $M$. Only $M$ is sorted, after which the exact scan is applied with the augmented coefficients in \Cref{eq:aug}.

\begin{table}[!htbp]
\centering
\small
\caption{Analytical VJP error against centered finite differences from an independent 200-step FP64 bisection reference. The relative-error tolerance is $10^{-3}$.}
\label{tab:gradcheck}
\begin{tabular}{rrrrrr}
\toprule
$B$ & $n$ & $k$ & $\alpha$ & Abs. error & Rel. error \\
\midrule
2 & 32 & 4 & -1.00 & 3.15e-11 & 8.01e-11 \\
1 & 64 & 8 & -0.50 & 1.19e-10 & 7.37e-11 \\
2 & 64 & 16 & -2.00 & 9.67e-11 & 2.30e-10 \\
\bottomrule
\end{tabular}
\end{table}

\noindent\textbf{Certification.}
The endpoint test accepts a bracket when $\Psi_r(a_R)\le k\le\Psi_r(a_L)$, using an FP32 tolerance of $10^{-4}\max(k,1)$. If the test fails, the margin $z$ is doubled and the bracket is recomputed. After four failed attempts, the exact full-sort solver is used. The operator comparisons in \Cref{sec:fast-exp,fig:baseline-plot-fb,tab:ks-main,tab:dftopk} measure the complete certified path, including retries and fallback when triggered. Initial bracket coverage and retry rates are reported separately in \Cref{sec:bracket-exp}.

\noindent\textbf{Numerical precision.}
The experiments use FP32. FP64 is used only for the finite-difference reference in \Cref{tab:gradcheck}. For the interval root, we select between the two forms in \Cref{eq:root-stable} according to the sign of $i^\star-1-k$. This avoids subtractive cancellation and remains well-defined as $A_{i^\star}\to0$. 

\section{Additional Speed and Ablation Results}
\label{app:speed-extra}

\begin{figure}[!htbp]
\centering
\includegraphics[width=.82\columnwidth]{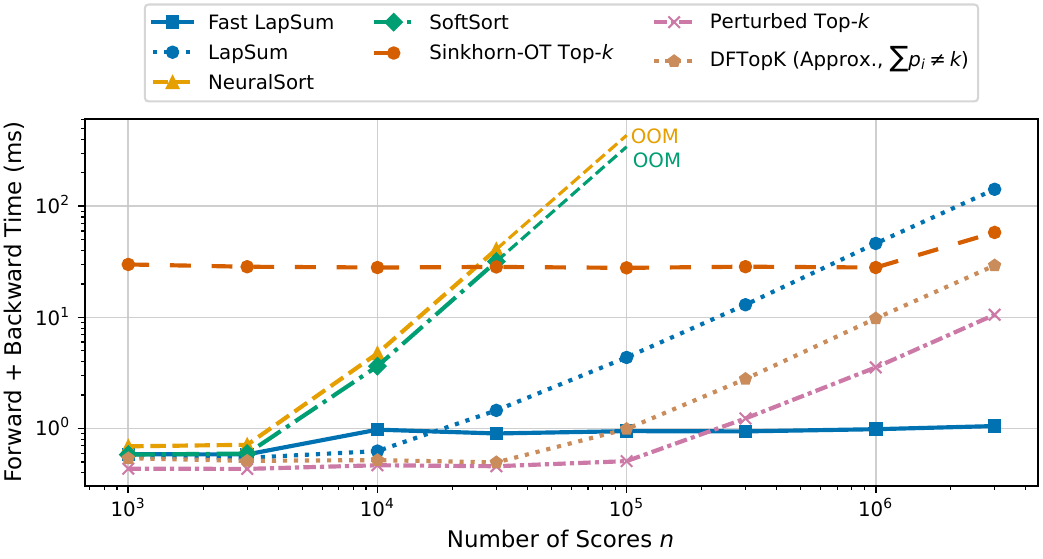}
\caption{Median forward and backward runtime on an NVIDIA GH200 120GB ($B{=}1$, $k{=}n/16$, FP32). The Fast LapSum curve includes endpoint certification and any retries.}
\label{fig:baseline-plot-fb}
\end{figure}

\Cref{fig:baseline-plot-fb} includes the analytical vector--Jacobian product in \Cref{eq:vjp}. At $n=3{\times}10^4$, the certified median forward and backward time is $1.344$\,ms for Fast LapSum, $40.952$\,ms for NeuralSort, and $31.839$\,ms for SoftSort.

\begin{table}[!htbp]
\centering
\scriptsize
\resizebox{\linewidth}{!}{%
\begin{tabular}{@{}crrrrrr@{}}
\toprule
$N$ & LapSum (ms) & Certified Fast LapSum (ms) & Certified p95 (ms) & Retry (\%) & Fallback (\%) & Speedup over LapSum \\
\midrule
$1{\times}10^6$ & 41.463 & \textbf{0.923} & 0.972 & 0.0 & 0.0 & $44.9\times$ \\
$3{\times}10^6$ & 127.739 & \textbf{1.008} & 1.052 & 0.0 & 0.0 & $126.7\times$ \\
$1{\times}10^7$ & 432.363 & \textbf{1.390} & 2.291 & 6.7 & 0.0 & $311.1\times$ \\
$3{\times}10^7$ & 1409.621 & \textbf{2.818} & 6.027 & 29.2 & 0.0 & $500.2\times$ \\
$1{\times}10^8$ & 5240.809 & \textbf{7.241} & 15.576 & 33.3 & 0.0 & $723.8\times$ \\
$3{\times}10^8$ & 16468.881 & \textbf{19.498} & 65.935 & 45.0 & 0.0 & $844.6\times$ \\
$1{\times}10^9$ & 55591.469 & \textbf{121.764} & 232.446 & 60.0 & 0.0 & $456.5\times$ \\
\bottomrule
\end{tabular}%
}
\caption{Forward runtime on an NVIDIA GH200 120GB ($B{=}1$, $k{=}N/16$, $\alpha{=}-1$). Fast LapSum timings include endpoint certification and retries. Retry is the percentage of timed calls requiring more than one sampled bracket. Fallback is the percentage reaching the exact full-sort path. Speedup is the LapSum median divided by the certified Fast LapSum median. The $456.5\times$ value at $N=10^9$ is reported as $457\times$ in the main text. Bold marks the lower median runtime.}
\label{tab:ks-main}
\end{table}

\subsection{Comparison with DFTopK}
\label{app:dftopk}

DFTopK~\citep{dftopk2025} also studies linear-time differentiable top-$k$. Its application experiments use RecFlow and an internal advertising system, so we restrict the comparison to the operator benchmark from its Table~2.

\paragraph{Protocol.}
We measure one forward and backward pass with $k=\lfloor n/2\rfloor$ and batch size $1$. All methods run on the same NVIDIA GH200 120GB in FP32. The comparison includes DFTopK, the official LapSum implementation~\citep{lapsum2025}, Fast LapSum, and the differentiable-ordering baselines from \Cref{fig:baseline-plot}. The original DFTopK table stops at $n=10^3$. We report sizes up to $n=3{\times}10^6$.

\begin{table}[!htbp]
\centering
\caption{Median forward and backward runtime in milliseconds on an NVIDIA GH200 120GB ($k=\lfloor n/2\rfloor$, batch size $1$, FP32). The Fast LapSum measurements include endpoint certification and retries. ``Mem.'' indicates that the estimated quadratic live tensors exceed 85\% of device memory. Bold marks the lowest runtime in each row.}
\label{tab:dftopk}
\small
\resizebox{\linewidth}{!}{%
\begin{tabular}{rrrrrrrr}
\toprule
$n$ & DFTopK & Fast LapSum & LapSum & NeuralSort & SoftSort & Sinkhorn-OT & Perturbed \\
\midrule
$10^{3}$ & 0.510 & 0.626 & 0.560 & 0.662 & 0.557 & 29.830 & \textbf{0.420} \\
$10^{4}$ & 0.476 & 1.586 & 0.634 & 4.906 & 3.796 & 29.202 & \textbf{0.454} \\
$10^{5}$ & 0.795 & 1.651 & 4.374 & Mem. & Mem. & 28.938 & \textbf{0.556} \\
$10^{6}$ & 6.554 & \textbf{2.555} & 46.579 & Mem. & Mem. & 28.686 & 3.877 \\
$3{\times}10^{6}$ & 23.087 & \textbf{4.345} & 144.005 & Mem. & Mem. & 58.561 & 11.528 \\
\bottomrule
\end{tabular}}
\end{table}

\begin{table}[!htbp]
\centering
\small
\setlength{\tabcolsep}{4pt}
\begin{tabular}{@{}l@{\;}ccccc@{}}
\toprule
\backslashbox{Operator}{Size} & $10^{3}$ & $10^{4}$ & $10^{5}$ & $10^{6}$ & $10^{7}$ \\
\midrule
Fast LapSum (ours) & $6.79\mathrm{e}{-8}$ & $8.17\mathrm{e}{-8}$ & $1.45\mathrm{e}{-7}$ & $2.62\mathrm{e}{-8}$ & $6.52\mathrm{e}{-8}$ \\
LapSum~\citep{lapsum2025} & $1.13\mathrm{e}{-7}$ & $2.34\mathrm{e}{-8}$ & $1.46\mathrm{e}{-6}$ & $1.12\mathrm{e}{-5}$ & $1.54\mathrm{e}{-3}$ \\
DFTopK~\citep{dftopk2025} & $2.39$ & $2.44$ & $2.46$ & $2.46$ & $2.46$ \\
NeuralSort~\citep{neuralsort2019} & $1.99\mathrm{e}{-9}$ & $2.04\mathrm{e}{-9}$ & Mem. & Mem. & Mem. \\
SoftSort~\citep{softsort2020} & \textbf{$1.30\mathrm{e}{-9}$} & \textbf{$2.71\mathrm{e}{-10}$} & Mem. & Mem. & Mem. \\
Sinkhorn-OT~\citep{xie2020softtopk} & $2.41\mathrm{e}{-8}$ & $2.11\mathrm{e}{-8}$ & $6.92\mathrm{e}{-8}$ & $3.26\mathrm{e}{-8}$ & $4.40\mathrm{e}{-8}$ \\
Perturbed~\citep{berthet2020perturbed} & $2.68\mathrm{e}{-8}$ & $2.67\mathrm{e}{-8}$ & \textbf{$2.57\mathrm{e}{-8}$} & \textbf{$2.58\mathrm{e}{-8}$} & \textbf{$2.58\mathrm{e}{-8}$} \\
\bottomrule
\end{tabular}%
\caption{Relative budget error $|\sum_i p_i-k|/k$ on an NVIDIA GH200 120GB ($k{=}N/16$, $\alpha{=}-1$, Gaussian scores). Fast LapSum uses the complete certified path. ``Mem.'' indicates that the estimated quadratic live tensors exceed 85\% of device memory. DFTopK does not impose $\sum_i p_i=k$. Bold marks the lowest measured error in each column.}
\label{tab:budget}
\end{table}

\paragraph{Results.}
\Cref{tab:dftopk} reports the measured runtimes. At $n=10^3$, DFTopK, certified Fast LapSum, and LapSum take $0.510$, $0.626$, and $0.560$\,ms, respectively. At $n=3{\times}10^6$, their runtimes are $23.087$, $4.345$, and $144.005$\,ms. Fast LapSum is therefore $5.3\times$ faster than DFTopK and $33.1\times$ faster than LapSum at the largest common size. NeuralSort and SoftSort exceed the memory criterion from $n=10^5$ because they construct quadratic tensors. DFTopK does not impose the LapSum budget equation $\sum_i p_i=k$. The resulting budget error is reported in \Cref{tab:budget}.

\section{Application Protocols}
\label{app:applications}

\subsection{3D Gaussian Splatting}

\paragraph{Standard scenes.}
For each scene and seed, the full ADC run gives the reference count \(N_{\mathrm{ADC}}\). We set \(K=\operatorname{round}(0.3N_{\mathrm{ADC}})\). ADC runs through iteration \(10{,}000\), after which densification is disabled. From iterations \(10{,}000\) to \(12{,}000\), the soft budget follows a cubic schedule from the active count \(N\) to \(K\), and the temperature decreases from 3.0 to 0.1. The scores are optimized with Adam at learning rate 0.02. No Gaussians are physically removed during this interval. At iteration \(12{,}000\), the mask is recomputed after the optimizer update, the highest-scoring \(K\) Gaussians are retained, and the selector is removed. The retained set is trained until iteration \(30{,}000\).

The footprint baseline uses the same activation iteration, value of \(K\), hard-pruning iteration, and recovery period. Its ranking is determined by the accumulated footprint statistic. The retention sweep uses bicycle, garden, kitchen, playroom, and train with seed 0. The final comparison uses all 13 scenes with seeds 0, 1, and 2.

During the soft stage, we store fixed-bin histograms of the mask probabilities on $[0,1]$ and their quantiles at fixed checkpoints. A final snapshot is taken after the mask recomputation at iteration \(12{,}000\) and before hard top-$K$ pruning. The same logging is used before the large-scene transition at iteration \(36{,}000\). Diagnostic logging is excluded from the reported training time.

\paragraph{Large scenes.}
The monolithic runs use the same 30\% retention target and score optimizer. Densification ends at iteration \(30{,}000\), score optimization ends with hard pruning at iteration \(36{,}000\), and training ends at iteration \(60{,}000\). These runs do not use spatial blocks or CPU offloading. The LapSum comparison changes only the threshold solver.

\paragraph{Measurements.}
PSNR, SSIM, and LPIPS are evaluated on held-out views. Training time is measured with GPU synchronization and excludes metric evaluation and diagnostic logging. Rendering throughput is measured from the saved model, and model size is the size of the final PLY file.

\paragraph{Per-scene results.}
\Cref{tab:gs-scene-quality,tab:gs-scene-resources} give the complete results underlying the 13-scene averages in \Cref{tab:gs-standard-quality}. The two pruning methods use the same per-scene Gaussian budget.

\begin{table*}[p]
\centering
\scriptsize
\setlength{\tabcolsep}{6pt}
\renewcommand{\arraystretch}{0.94}
\caption{Per-scene image quality at 30\% retention. Results are means $\pm$ standard deviations over three seeds. Bold marks the best mean within each scene.}
\label{tab:gs-scene-quality}
\begin{tabular}{@{}llrrr@{}}
\toprule
Scene & Method & PSNR (dB) $\uparrow$ & SSIM $\uparrow$ & LPIPS $\downarrow$ \\
\midrule
Bicycle & ADC & \textbf{\boldmath$25.27\pm0.03$} & \textbf{\boldmath$0.767\pm0.000$} & \textbf{\boldmath$0.208\pm0.000$} \\
 & \quad+ hard pruning & $25.21\pm0.04$ & $0.757\pm0.000$ & $0.233\pm0.001$ \\
 & \quad+ Fast LapSum & $25.25\pm0.06$ & $0.761\pm0.001$ & $0.232\pm0.001$ \\
\addlinespace[1pt]
Flowers & ADC & \textbf{\boldmath$21.58\pm0.00$} & \textbf{\boldmath$0.605\pm0.001$} & \textbf{\boldmath$0.336\pm0.000$} \\
 & \quad+ hard pruning & $21.48\pm0.01$ & $0.596\pm0.001$ & $0.357\pm0.001$ \\
 & \quad+ Fast LapSum & $21.49\pm0.04$ & $0.597\pm0.000$ & $0.354\pm0.000$ \\
\addlinespace[1pt]
Garden & ADC & \textbf{\boldmath$27.44\pm0.04$} & \textbf{\boldmath$0.867\pm0.001$} & \textbf{\boldmath$0.107\pm0.000$} \\
 & \quad+ hard pruning & $26.90\pm0.03$ & $0.847\pm0.001$ & $0.139\pm0.000$ \\
 & \quad+ Fast LapSum & $27.41\pm0.05$ & $0.865\pm0.001$ & $0.116\pm0.000$ \\
\addlinespace[1pt]
Stump & ADC & $26.60\pm0.01$ & $0.771\pm0.001$ & \textbf{\boldmath$0.216\pm0.001$} \\
 & \quad+ hard pruning & $26.56\pm0.05$ & $0.770\pm0.001$ & $0.225\pm0.001$ \\
 & \quad+ Fast LapSum & \textbf{\boldmath$26.63\pm0.04$} & \textbf{\boldmath$0.773\pm0.001$} & $0.229\pm0.001$ \\
\addlinespace[1pt]
Treehill & ADC & $22.50\pm0.07$ & \textbf{\boldmath$0.634\pm0.001$} & \textbf{\boldmath$0.324\pm0.001$} \\
 & \quad+ hard pruning & $22.56\pm0.07$ & $0.630\pm0.001$ & $0.350\pm0.001$ \\
 & \quad+ Fast LapSum & \textbf{\boldmath$22.67\pm0.04$} & $0.634\pm0.001$ & $0.349\pm0.000$ \\
\addlinespace[1pt]
Room & ADC & \textbf{\boldmath$31.57\pm0.09$} & \textbf{\boldmath$0.920\pm0.000$} & \textbf{\boldmath$0.217\pm0.001$} \\
 & \quad+ hard pruning & $31.21\pm0.16$ & $0.912\pm0.000$ & $0.236\pm0.001$ \\
 & \quad+ Fast LapSum & $31.49\pm0.10$ & $0.919\pm0.000$ & $0.222\pm0.001$ \\
\addlinespace[1pt]
Counter & ADC & \textbf{\boldmath$29.10\pm0.00$} & \textbf{\boldmath$0.909\pm0.000$} & \textbf{\boldmath$0.199\pm0.000$} \\
 & \quad+ hard pruning & $28.63\pm0.03$ & $0.894\pm0.000$ & $0.228\pm0.000$ \\
 & \quad+ Fast LapSum & $28.96\pm0.03$ & $0.906\pm0.000$ & $0.207\pm0.000$ \\
\addlinespace[1pt]
Kitchen & ADC & \textbf{\boldmath$31.44\pm0.07$} & \textbf{\boldmath$0.928\pm0.000$} & \textbf{\boldmath$0.125\pm0.000$} \\
 & \quad+ hard pruning & $31.08\pm0.04$ & $0.921\pm0.000$ & $0.139\pm0.000$ \\
 & \quad+ Fast LapSum & $31.26\pm0.24$ & $0.927\pm0.001$ & $0.129\pm0.000$ \\
\addlinespace[1pt]
Bonsai & ADC & \textbf{\boldmath$32.27\pm0.11$} & \textbf{\boldmath$0.942\pm0.001$} & \textbf{\boldmath$0.203\pm0.000$} \\
 & \quad+ hard pruning & $31.77\pm0.10$ & $0.934\pm0.001$ & $0.220\pm0.000$ \\
 & \quad+ Fast LapSum & $31.97\pm0.04$ & $0.938\pm0.000$ & $0.211\pm0.000$ \\
\addlinespace[1pt]
Train & ADC & $22.20\pm0.14$ & \textbf{\boldmath$0.823\pm0.001$} & \textbf{\boldmath$0.195\pm0.000$} \\
 & \quad+ hard pruning & $21.93\pm0.07$ & $0.800\pm0.001$ & $0.240\pm0.001$ \\
 & \quad+ Fast LapSum & \textbf{\boldmath$22.22\pm0.07$} & $0.817\pm0.001$ & $0.216\pm0.001$ \\
\addlinespace[1pt]
Truck & ADC & \textbf{\boldmath$25.42\pm0.02$} & \textbf{\boldmath$0.885\pm0.000$} & \textbf{\boldmath$0.142\pm0.000$} \\
 & \quad+ hard pruning & $25.13\pm0.07$ & $0.878\pm0.000$ & $0.156\pm0.000$ \\
 & \quad+ Fast LapSum & $25.40\pm0.01$ & $0.884\pm0.001$ & $0.147\pm0.001$ \\
\addlinespace[1pt]
Dr. Johnson & ADC & $29.44\pm0.06$ & $0.905\pm0.000$ & \textbf{\boldmath$0.236\pm0.000$} \\
 & \quad+ hard pruning & $29.24\pm0.08$ & $0.898\pm0.000$ & $0.260\pm0.000$ \\
 & \quad+ Fast LapSum & \textbf{\boldmath$29.49\pm0.02$} & \textbf{\boldmath$0.907\pm0.000$} & $0.242\pm0.001$ \\
\addlinespace[1pt]
Playroom & ADC & \textbf{\boldmath$30.20\pm0.13$} & $0.909\pm0.001$ & \textbf{\boldmath$0.240\pm0.001$} \\
 & \quad+ hard pruning & $29.98\pm0.05$ & $0.906\pm0.001$ & $0.256\pm0.000$ \\
 & \quad+ Fast LapSum & $30.19\pm0.06$ & \textbf{\boldmath$0.910\pm0.000$} & $0.247\pm0.000$ \\
\bottomrule
\end{tabular}
\end{table*}

\begin{table*}[p]
\centering
\scriptsize
\setlength{\tabcolsep}{5pt}
\renewcommand{\arraystretch}{0.94}
\caption{Per-scene representation size and runtime at 30\% retention. Results are means $\pm$ standard deviations over three seeds. Bold marks the best mean within each scene. Equal values arise because the two pruning methods retain the same exact number of Gaussians.}
\label{tab:gs-scene-resources}
\begin{tabular}{@{}llrrrr@{}}
\toprule
Scene & Method & $N$ (M) $\downarrow$ & Train (min) $\downarrow$ & FPS $\uparrow$ & Model (MiB) $\downarrow$ \\
\midrule
Bicycle & ADC & $4.88\pm0.01$ & $14.19\pm0.14$ & $116.3\pm0.3$ & $1155.3\pm3.3$ \\
 & \quad+ hard pruning & \textbf{\boldmath$1.47\pm0.01$} & $10.65\pm0.04$ & $182.4\pm2.7$ & \textbf{\boldmath$348.2\pm1.3$} \\
 & \quad+ Fast LapSum & \textbf{\boldmath$1.47\pm0.01$} & \textbf{\boldmath$9.61\pm0.08$} & \textbf{\boldmath$317.6\pm2.2$} & \textbf{\boldmath$348.2\pm1.3$} \\
\addlinespace[1pt]
Flowers & ADC & $2.89\pm0.02$ & $10.36\pm0.10$ & $234.7\pm5.7$ & $682.8\pm4.3$ \\
 & \quad+ hard pruning & \textbf{\boldmath$0.87\pm0.00$} & $8.34\pm0.18$ & $371.8\pm1.7$ & \textbf{\boldmath$206.4\pm0.3$} \\
 & \quad+ Fast LapSum & \textbf{\boldmath$0.87\pm0.00$} & \textbf{\boldmath$8.15\pm0.04$} & \textbf{\boldmath$455.5\pm9.1$} & \textbf{\boldmath$206.4\pm0.3$} \\
\addlinespace[1pt]
Garden & ADC & $4.15\pm0.01$ & $13.07\pm0.03$ & $162.8\pm0.5$ & $981.2\pm1.3$ \\
 & \quad+ hard pruning & \textbf{\boldmath$1.24\pm0.00$} & $10.09\pm0.08$ & $264.6\pm1.5$ & \textbf{\boldmath$294.0\pm0.8$} \\
 & \quad+ Fast LapSum & \textbf{\boldmath$1.24\pm0.00$} & \textbf{\boldmath$9.49\pm0.01$} & \textbf{\boldmath$352.3\pm11.1$} & \textbf{\boldmath$294.0\pm0.8$} \\
\addlinespace[1pt]
Stump & ADC & $4.30\pm0.04$ & $10.96\pm0.07$ & $194.3\pm2.7$ & $1015.8\pm9.8$ \\
 & \quad+ hard pruning & \textbf{\boldmath$1.27\pm0.02$} & $8.24\pm0.04$ & $314.7\pm0.4$ & \textbf{\boldmath$301.4\pm4.9$} \\
 & \quad+ Fast LapSum & \textbf{\boldmath$1.27\pm0.02$} & \textbf{\boldmath$7.90\pm0.12$} & \textbf{\boldmath$433.7\pm11.6$} & \textbf{\boldmath$301.4\pm4.9$} \\
\addlinespace[1pt]
Treehill & ADC & $3.28\pm0.02$ & $11.33\pm0.07$ & $200.0\pm2.3$ & $775.7\pm5.0$ \\
 & \quad+ hard pruning & \textbf{\boldmath$0.98\pm0.01$} & $8.90\pm0.08$ & $316.6\pm1.9$ & \textbf{\boldmath$231.9\pm1.2$} \\
 & \quad+ Fast LapSum & \textbf{\boldmath$0.98\pm0.01$} & \textbf{\boldmath$8.60\pm0.12$} & \textbf{\boldmath$415.2\pm4.7$} & \textbf{\boldmath$231.9\pm1.2$} \\
\addlinespace[1pt]
Room & ADC & $1.31\pm0.00$ & $14.92\pm0.07$ & $195.5\pm0.6$ & $309.8\pm1.1$ \\
 & \quad+ hard pruning & \textbf{\boldmath$0.40\pm0.00$} & $12.95\pm0.14$ & $262.8\pm8.2$ & \textbf{\boldmath$93.5\pm0.4$} \\
 & \quad+ Fast LapSum & \textbf{\boldmath$0.40\pm0.00$} & \textbf{\boldmath$12.90\pm0.14$} & \textbf{\boldmath$318.7\pm1.4$} & \textbf{\boldmath$93.5\pm0.4$} \\
\addlinespace[1pt]
Counter & ADC & $1.08\pm0.00$ & $13.93\pm0.07$ & $208.5\pm1.8$ & $255.0\pm0.9$ \\
 & \quad+ hard pruning & \textbf{\boldmath$0.32\pm0.00$} & $12.29\pm0.10$ & $282.6\pm2.8$ & \textbf{\boldmath$76.3\pm0.7$} \\
 & \quad+ Fast LapSum & \textbf{\boldmath$0.32\pm0.00$} & \textbf{\boldmath$12.14\pm0.08$} & \textbf{\boldmath$337.5\pm4.3$} & \textbf{\boldmath$76.3\pm0.7$} \\
\addlinespace[1pt]
Kitchen & ADC & $1.61\pm0.01$ & $16.33\pm0.12$ & $178.9\pm1.5$ & $380.4\pm1.5$ \\
 & \quad+ hard pruning & \textbf{\boldmath$0.48\pm0.00$} & $13.78\pm0.04$ & $264.5\pm1.1$ & \textbf{\boldmath$113.9\pm0.7$} \\
 & \quad+ Fast LapSum & \textbf{\boldmath$0.48\pm0.00$} & \textbf{\boldmath$13.22\pm0.04$} & \textbf{\boldmath$334.7\pm3.4$} & \textbf{\boldmath$113.9\pm0.7$} \\
\addlinespace[1pt]
Bonsai & ADC & $1.07\pm0.00$ & $11.43\pm0.12$ & $289.4\pm5.4$ & $254.2\pm1.0$ \\
 & \quad+ hard pruning & \textbf{\boldmath$0.32\pm0.00$} & \textbf{\boldmath$10.39\pm0.04$} & $353.6\pm13.0$ & \textbf{\boldmath$76.0\pm0.4$} \\
 & \quad+ Fast LapSum & \textbf{\boldmath$0.32\pm0.00$} & $10.39\pm0.10$ & \textbf{\boldmath$412.4\pm9.0$} & \textbf{\boldmath$76.0\pm0.4$} \\
\addlinespace[1pt]
Train & ADC & $1.09\pm0.01$ & $8.60\pm0.04$ & $223.0\pm0.7$ & $258.2\pm1.9$ \\
 & \quad+ hard pruning & \textbf{\boldmath$0.33\pm0.00$} & \textbf{\boldmath$6.64\pm0.02$} & $413.2\pm6.2$ & \textbf{\boldmath$77.3\pm0.7$} \\
 & \quad+ Fast LapSum & \textbf{\boldmath$0.33\pm0.00$} & $6.64\pm0.06$ & \textbf{\boldmath$455.6\pm25.1$} & \textbf{\boldmath$77.3\pm0.7$} \\
\addlinespace[1pt]
Truck & ADC & $2.04\pm0.01$ & $8.65\pm0.03$ & $233.3\pm3.2$ & $482.9\pm2.7$ \\
 & \quad+ hard pruning & \textbf{\boldmath$0.61\pm0.00$} & $6.98\pm0.02$ & $358.7\pm13.8$ & \textbf{\boldmath$144.9\pm1.1$} \\
 & \quad+ Fast LapSum & \textbf{\boldmath$0.61\pm0.00$} & \textbf{\boldmath$6.72\pm0.04$} & \textbf{\boldmath$448.8\pm21.6$} & \textbf{\boldmath$144.9\pm1.1$} \\
\addlinespace[1pt]
Dr. Johnson & ADC & $3.16\pm0.03$ & $16.72\pm0.05$ & $155.8\pm2.4$ & $746.4\pm8.3$ \\
 & \quad+ hard pruning & \textbf{\boldmath$0.95\pm0.01$} & $13.33\pm0.12$ & $270.1\pm0.7$ & \textbf{\boldmath$224.2\pm2.6$} \\
 & \quad+ Fast LapSum & \textbf{\boldmath$0.95\pm0.01$} & \textbf{\boldmath$13.00\pm0.15$} & \textbf{\boldmath$313.0\pm4.5$} & \textbf{\boldmath$224.2\pm2.6$} \\
\addlinespace[1pt]
Playroom & ADC & $1.86\pm0.01$ & $13.21\pm0.09$ & $232.5\pm2.0$ & $439.5\pm3.3$ \\
 & \quad+ hard pruning & \textbf{\boldmath$0.56\pm0.00$} & $10.92\pm0.03$ & $377.7\pm8.9$ & \textbf{\boldmath$131.8\pm1.1$} \\
 & \quad+ Fast LapSum & \textbf{\boldmath$0.56\pm0.00$} & \textbf{\boldmath$10.86\pm0.09$} & \textbf{\boldmath$422.2\pm5.8$} & \textbf{\boldmath$131.8\pm1.1$} \\
\bottomrule
\end{tabular}
\end{table*}

\paragraph{Qualitative comparisons.}
\Cref{fig:gs-qualitative-1,fig:gs-qualitative-2,fig:gs-qualitative-3,fig:gs-qualitative-4,fig:gs-qualitative-5} compare ADC and Fast LapSum on all 13 scenes. For each scene, we select the ADC held-out view whose PSNR is closest to that scene's median per-view PSNR, then display the same view for both methods.

\begin{figure*}[p]
\centering
\includegraphics[width=\textwidth]{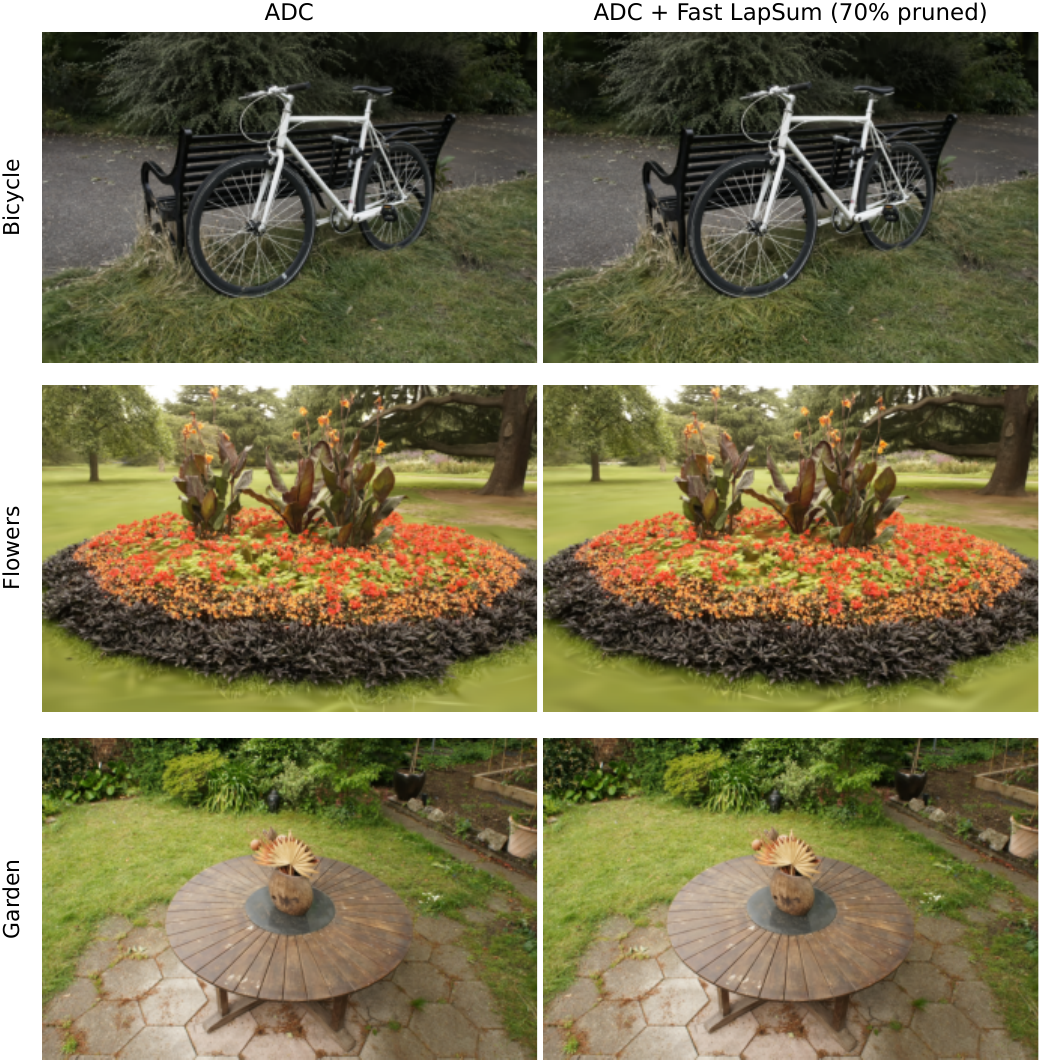}
\caption{Qualitative comparison on Bicycle, Flowers, and Garden. Fast LapSum retains 30\% of the ADC Gaussians.}
\label{fig:gs-qualitative-1}
\end{figure*}

\begin{figure*}[p]
\centering
\includegraphics[width=\textwidth]{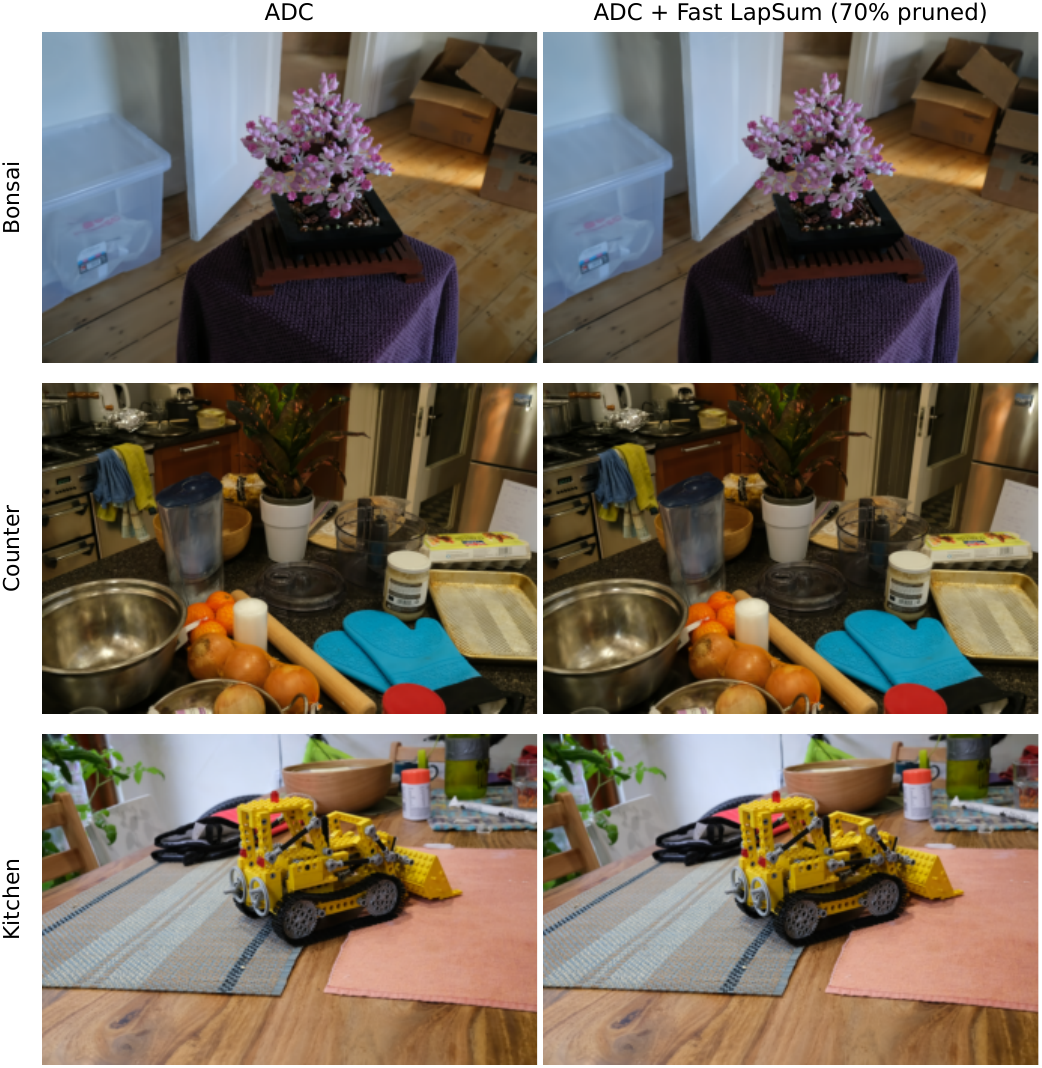}
\caption{Qualitative comparison on Bonsai, Counter, and Kitchen. Fast LapSum retains 30\% of the ADC Gaussians.}
\label{fig:gs-qualitative-2}
\end{figure*}

\begin{figure*}[p]
\centering
\includegraphics[width=\textwidth]{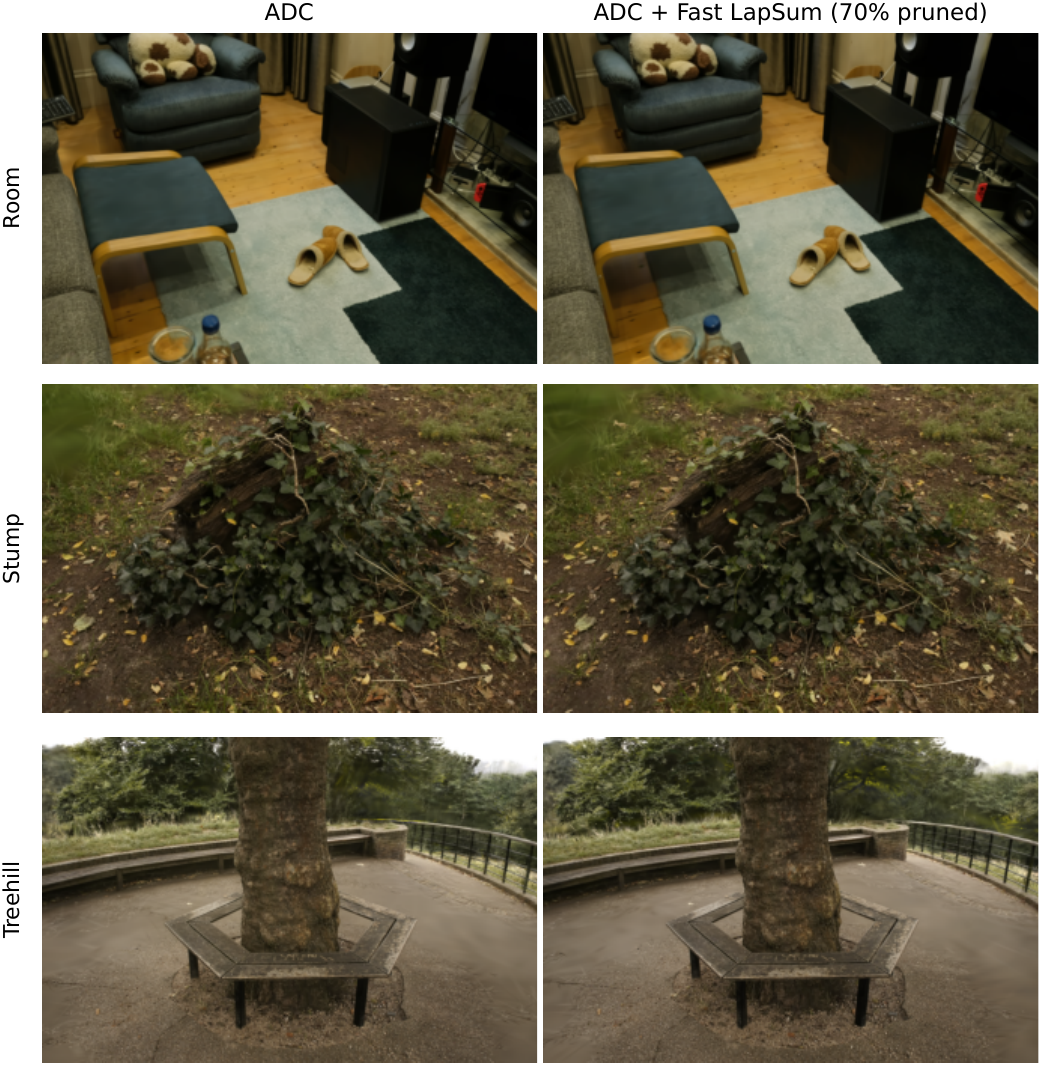}
\caption{Qualitative comparison on Room, Stump, and Treehill. Fast LapSum retains 30\% of the ADC Gaussians.}
\label{fig:gs-qualitative-3}
\end{figure*}

\begin{figure*}[p]
\centering
\includegraphics[width=\textwidth]{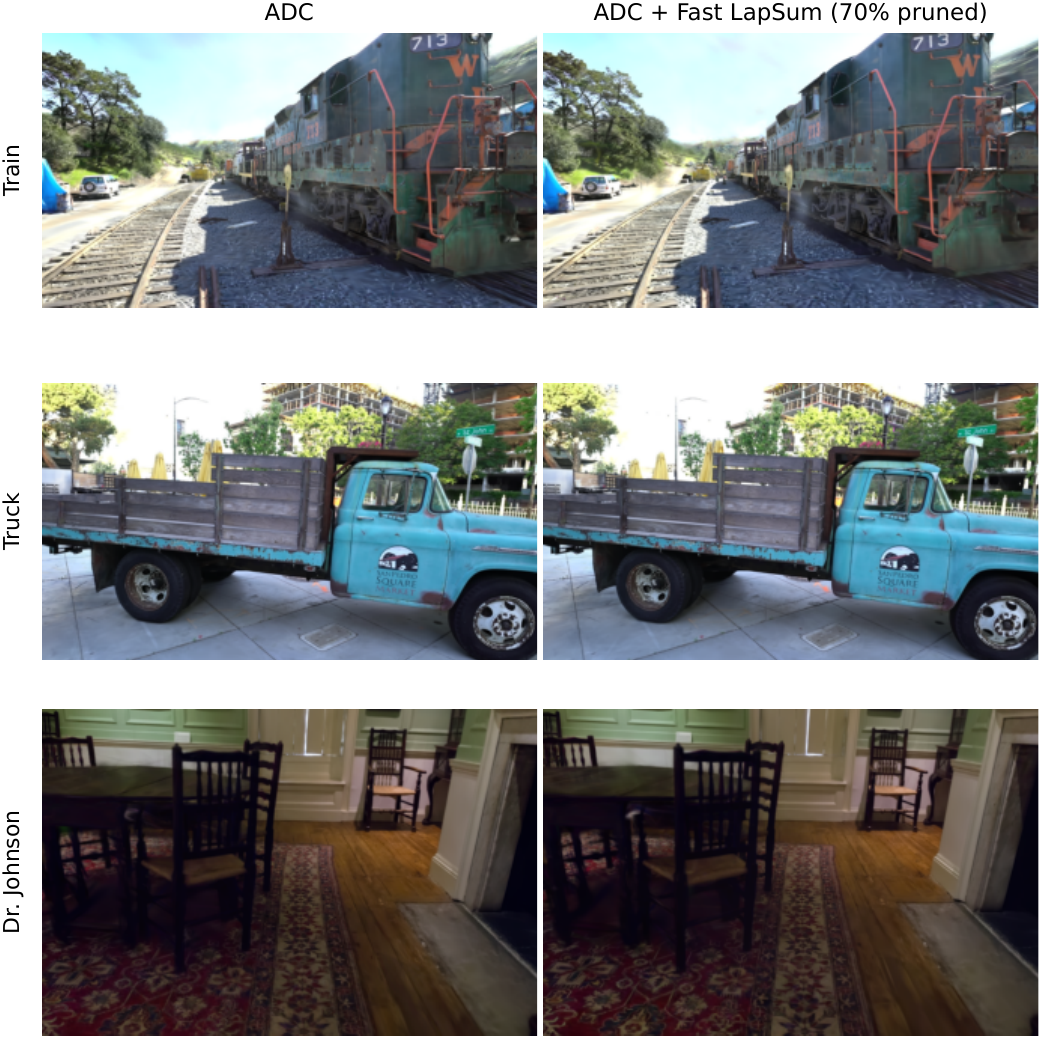}
\caption{Qualitative comparison on Train, Truck, and Dr. Johnson. Fast LapSum retains 30\% of the ADC Gaussians.}
\label{fig:gs-qualitative-4}
\end{figure*}

\begin{figure*}[p]
\centering
\includegraphics[width=\textwidth]{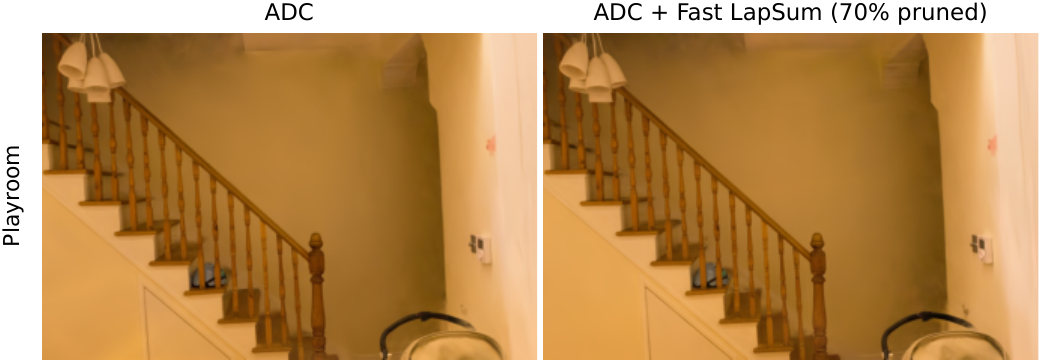}
\caption{Qualitative comparison on Playroom. Fast LapSum retains 30\% of the ADC Gaussians.}
\label{fig:gs-qualitative-5}
\end{figure*}

\newpage
\subsection{Exact-Budget Adversarial Attacks}

For an image \(x\in[0,1]^{3\times H\times W}\), the attack optimizes one score per pixel, \(r\in\mathbb{R}^{HW}\), and replacement-color logits \(z\). Let $p$ denote the soft mask and $h$ the hard top-$K$ mask:
\begin{equation}
  h=\mathbf{1}[r\in\operatorname{top}\text{-}K],\qquad
  \widehat h=p+\operatorname{sg}(h-p),\qquad
  p=\Lap_K(r;\alpha),
\end{equation}
and the attacked image is
\begin{equation}
  x'=x\odot(1-\widehat h)+\sigma(z)\odot\widehat h.
\end{equation}
The forward pass uses $\widehat h=h$ and therefore selects exactly \(K\) pixel locations. The backward pass uses the derivative of the soft mask $p$. If an optimized replacement color equals the original color, the measured \(L_0\) can be smaller than \(K\). We minimize the untargeted Carlini--Wagner margin
\begin{equation}
  \mathcal{L}_{\mathrm{CW}}=f_y(x')-\max_{c\ne y}f_c(x').
\end{equation}
A negative margin indicates misclassification. We initialize \(r\) with Gaussian noise of standard deviation \(10^{-3}\) and initialize \(z\) from the input with Gaussian noise of standard deviation 0.05. Both variables are optimized with Adam. The temperature decreases geometrically from 1.0 to 0.05, and the color learning rate is 0.05. On a separate set of 100 images, we selected score learning rates of 0.01 for ConvNeXt-V2 Tiny and 0.05 for ResNet-50. Each evaluation uses at most 500 steps and three restarts. Success is checked after quantizing the adversarial image to 8-bit color.

The ImageNet evaluation contains 1,000 validation images classified correctly by both frozen models. We use \(K\in\{1,5,10,50,100,500\}\). Sparse-PGD and PGD-\(L_0\) use the same images, models, budgets, and 500-step limit. For DIV2K, we sort the 900 training and validation images by filename and select the first 15 for which ConvNeXt-V2 assigns at least 0.8 probability to its top-1 prediction. The first five are used for protocol development, and the remaining ten form the test set. The clean top-1 prediction is used as a pseudolabel. We evaluate \(K\in\{100,500,1000,1500,2000\}\).

\end{document}